\documentclass[runningheads]{llncs}

\usepackage[FINAL,year=2026,ID=4529]{eccv}

\usepackage{eccvabbrv}

\usepackage{graphicx}
\usepackage{booktabs}

\usepackage[accsupp]{axessibility}  
\usepackage{booktabs}

\usepackage{arydshln}

\newtheorem{assumption}{Assumption}
\usepackage{graphicx}
\usepackage{tabularx}

\usepackage[table]{xcolor}

\usepackage{booktabs}
\usepackage{multirow}
\usepackage{array}
\usepackage{color}
\usepackage{colortbl} 
\usepackage{siunitx}  
\usepackage{wrapfig}
\usepackage[table]{xcolor}
\usepackage{arydshln}
\usepackage[utf8]{inputenc}
\usepackage{pifont}
\usepackage{algorithm}
\usepackage{algorithmic}
\DeclareUnicodeCharacter{2717}{\ding{55}}
\definecolor{graybg}{gray}{0.95}
\usepackage{xcolor}
\usepackage{adjustbox}
\usepackage[most]{tcolorbox}
\DeclareUnicodeCharacter{2713}{\ding{51}}
\usepackage{xcolor}
\definecolor{LightPink}{RGB}{255, 182, 193}
\newcolumntype{C}{>{\centering\arraybackslash}X}

\usepackage{hyperref}

\usepackage{orcidlink}

\usepackage{makecell}
\begin{document}

\title{Geometric Regularization for Long-Tailed Semi-Supervised Learning via Gaussian Feature Bridges
} 


\titlerunning{Geometric Regularization via Gaussian Feature Bridges}

\author{
Hongyang He\inst{1,2}\orcidlink{0009-0001-4889-8118}
\textsuperscript{\ensuremath{\dagger}} \and
Xinyuan Song\inst{3}\orcidlink{0009-0005-4209-5671}
\textsuperscript{\ensuremath{\dagger}} \and
Yan Zhong\inst{4}\orcidlink{0000-0003-0005-2620} \and
Daizong Liu\inst{2}\orcidlink{0000-0001-8179-4508}\textsuperscript{*} \and
Yanbin Li\inst{4}\orcidlink{0009-0009-7501-9754} \and
Yang-fan He\orcidlink{0000-0001-7172-9403} \and
Wenqiao Zhang\inst{5}\orcidlink{0000-0002-5988-7609}
}

\authorrunning{H.~He et al.}

\institute{
University of Warwick, Coventry, United Kingdom \and
Wuhan University, Wuhan, China \and
Emory University, Atlanta, GA, USA \and
Peking University, Beijing, China \and
Zhejiang University, Hangzhou, China
}

\maketitle
\begingroup
\renewcommand{\thefootnote}{}
\footnotetext{
This work was completed during an internship at Maniforda.Ai.
}
\footnotetext{
\textsuperscript{\ensuremath{\dagger}} Hongyang He and Xinyuan Song contributed equally to this work.
}
\footnotetext{
\textsuperscript{*} Corresponding author: Daizong Liu.
}
\endgroup

\begin{abstract}
Real-world semi-supervised learning (SSL) often encounters significant challenges with long-tailed label distributions and noisy pseudo-labels, which hinder generalization and amplify confirmation bias. In this work, we introduce a novel framework, Gaussian Bridge Consistency (GBC), to address these challenges by constructing semantic interpolation paths between unlabeled samples and high-quality class anchors. Our method maintains a dynamic Prototype Atlas that stores a diverse and evolving set of labeled and pseudo-labeled exemplars per class. For each unlabeled instance, GBC forms a class-conditional Gaussian Feature Bridge in the latent space, enabling the student model to traverse a smooth trajectory from uncertain predictions to reliable class prototypes. A bridge consistency loss is applied along this path to enforce alignment with a geometrically interpolated target distribution. Furthermore, we propose BridgeMix, a confidence-aware feature mixing strategy that interpolates both sample and anchor pairs to amplify cross-sample generalization. Extensive experiments on CIFAR10-LT and ImageNet-LT (USB benchmarks) validate the robustness and effectiveness of GBC under realistic long-tailed SSL settings, consistently improving long tail-class performance without sacrificing scalability.
\end{abstract}

\section{Introduction}
\label{sec:intro}

Semi-supervised learning (SSL) has emerged as a promising paradigm to reduce annotation costs by leveraging large-scale unlabeled data alongside a small labeled subset \cite{A1,A2,A3,A4,A5}. Despite substantial progress, existing SSL methods often fall short in realistic scenarios characterized by severe class imbalance, long-tailed distributions, and noisy pseudo-labels \cite{A7,A8,A9,A10}. These challenges introduce significant confirmation bias, where dominant classes monopolize learning signals while minority (tail) classes suffer from undertraining and semantic drift \cite{A11,A12,A13}. This calls for a new perspective that explicitly models the uncertainty and transition dynamics between weakly predicted unlabeled instances and reliable class representations. \textbf{Our key motivation} stems from the question:
\begin{center}
\begin{tcolorbox}[colback=gray!5,colframe=gray!70!black,width=0.95\linewidth,boxrule=0.5pt,arc=2pt]
\centering
\emph{Can we transform sparse unlabeled features into reliable supervision signals by constructing smooth semantic pathways?}
\end{tcolorbox}
\end{center}

Inspired by the Schr\"odinger Bridge theory—a probabilistic framework \cite{A25} that interpolates between two distributions via optimal stochastic processes \cite{A14,A15,A16,A17}—we propose to treat labeled and pseudo-labeled class exemplars as endpoints of a semantic bridge, and encourage a student model to learn consistently along this interpolated path.

In this paper, we present a novel framework called \textbf{Gaussian Bridge Consistency (GBC)}, which introduces a simple yet efficient semantic interpolation mechanism into the SSL pipeline. Central to GBC is the notion of a \emph{Gaussian Feature Bridge}, a latent-space path that connects an unlabeled sample to a class-specific anchor retrieved from a dynamic Prototype Atlas (PA) \cite{A13,A18,A19,A21}. At multiple intermediate points along the bridge, we inject actionable supervision by enforcing bridge-consistency loss, which aligns the student's prediction with a soft geometric target interpolated between the sample and anchor distributions. This design encourages the model to traverse from uncertain to reliable features in a controllable manner, thereby reducing semantic drift and enhancing long tail-class robustness \cite{A22}.

To further improve the supervision signal and sample diversity, we introduce \textbf{BridgeMix}, a novel confidence-aware variant of MixUp. By interpolating not only features but also their corresponding anchors using pseudo-label confidence as a guidance weight, BridgeMix constructs smoothed bridge paths that exploit high-confidence samples to guide lower-confidence ones—all within the original loss formulation, with no additional training complexity \cite{A23,A24,A11,A12,A51,A52,A53,A54,A55,A56,A57,A58,A59,A60,A61,A62,A63,A64,A65,A66,A67}.

Extensive experiments on CIFAR10-LT, CIFAR100-LT, STL10-LT, and ImageNetLT benchmarks demonstrate that GBC significantly mitigates the performance collapse typical in realistic SSL conditions. By transforming uncertain feature states into reliable class-consistent representations, our framework enhances tail-class accuracy while maintaining compatibility across diverse architectures. Our core contributions are summarized as follows:

First, we identify the fundamental challenge in long-tailed semi-supervised learning (LTSSL) as the instability of feature alignment caused by biased pseudo-labels, which often drift toward majority classes. To address this, we propose GBC, a Schrödinger Bridge-inspired framework that regularizes the learning process by constructing class-conditional interpolation paths in latent space. This mechanism pulls noisy, strongly augmented unlabeled features toward reliable anchors stored in a dynamically updated Prototype Atlas (PA), ensuring stable semantic alignment even for data-scarce tail categories.

Second, we introduce BridgeMix, a confidence-guided manifold regularization strategy designed to strengthen bridge-level supervision. By employing a dynamic mixing coefficient derived from pseudo-label certainties, BridgeMix enables highly confident samples and anchors to guide the representation learning of uncertain ones. This process effectively regularizes the model's behavior across the sample-anchor trajectory and enhances the diversity of the bridge manifold without necessitating modifications to the underlying loss design.

Third, we establish a rigorous theoretical foundation for bridge consistency by demonstrating its role as an intrinsic geometric regularizer. We formally prove that GBC enforces geometric smoothness and local Lipschitz continuity on the decision boundary, which stabilizes the model against feature-level perturbations. Furthermore, we derive a generalization bound for BridgeMix, showing that it effectively tightens the true risk by reducing hypothesis complexity and minimizing the distributional discrepancy between mixed and empirical distributions.

Finally, our results demonstrate consistent and significant gains over prior state-of-the-art methods, establishing GBC as a robust, scalable, and mathematically grounded solution for semi-supervised learning under severe long-tailed conditions.

\section{Related Works}

\textbf{Semi-Supervised Learning.}
Classic SSL techniques rely heavily on pseudo-labeling \cite{A5,A3} and consistency regularization \cite{A26}, which aim to enhance generalization by enforcing prediction invariance across augmentations.
However, most SSL methods assume class-balanced labeled/unlabeled data, which often does not hold in practice.
As a result, pseudo-labeling can be error-prone and biased toward head classes, especially under low-label or long-tailed settings \cite{A10,A9}.
Recent efforts address this via class-balancing strategies such as reweighting, calibration, or pseudo-label refinement \cite{A10,A9}, but they are typically applied at the output/logit level and lack intermediate feature-level alignment.

\textbf{Long-Tailed SSL.}
To tackle the real-world challenges of class imbalance and distribution mismatch, Long-Tailed SSL (LTSSL) methods have been proposed.
Notably, ACR \cite{A8} introduces an adaptive consistency regularizer to refine pseudo-labels based on estimated class distributions.
Subsequent works improve robustness through weighting, debiasing, and expertization under the long-tailed regime, including SAW \cite{A28}, Adsh \cite{A29}, DePL \cite{A30}, BaCon \cite{A31}, CPE \cite{A32}, and SimPro \cite{A13}, which represents a powerful SOTA framework for Realistic LTSSL (ReaLTSSL).
Meta-Expert \cite{A33} goes beyond single/expert-ensemble heuristics via dynamic expert assignment and multi-depth feature fusion to mitigate head bias and reduce pseudo-label error.

\textbf{Interpolation-Based Learning and Consistency.}
Feature-level interpolation techniques such as MixUp \cite{A23}, Manifold MixUp \cite{A24} and ProbPseudo Mixup \cite{A27} have shown promise in regularizing training signals, and are often combined with pseudo-labeling/consistency objectives \cite{A3,A26,A5}.
However, these approaches still lack an explicit mechanism to align the semantic structure of features over uncertain unlabeled samples.
Our work bridges this gap by constructing a Gaussian path between sample–anchor pairs, enabling a student model to learn consistent predictions along a semantically meaningful trajectory in feature space.

\begin{figure*}[t]
\centering
\includegraphics[width=1\linewidth]{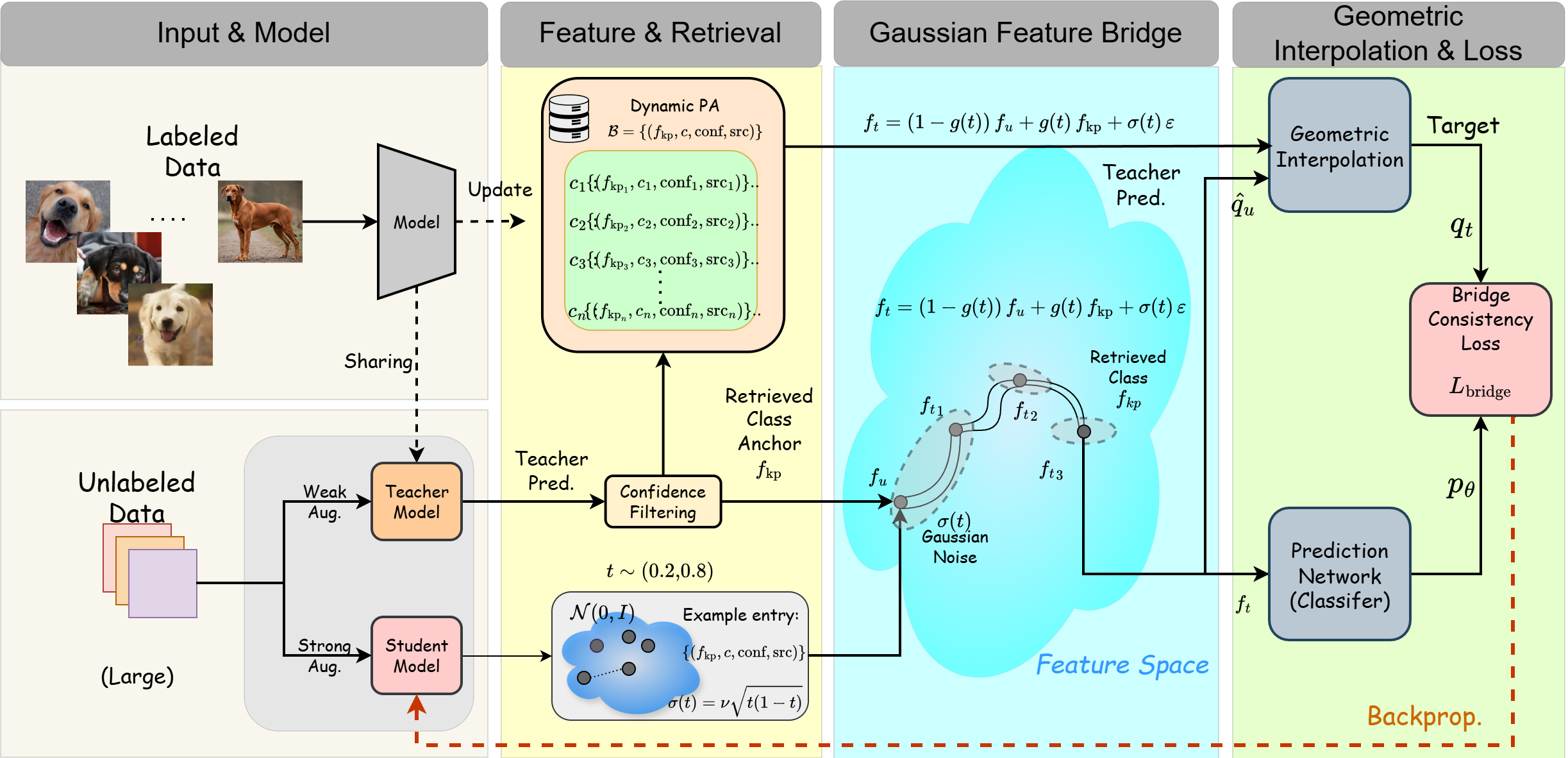}
\caption{
\textbf{Overview of our proposed framework.} 
}
\label{fig:framework}
\vspace{-1.2em}
\end{figure*}

\section{Method}
\label{sec:method}

\noindent 
\textbf{Framework Overview.} In ReaLTSSL, the labeled and unlabeled data often follow different and unknown class distributions, violating the standard i.i.d.\ assumption and causing instability in pseudo-label learning. 
To address this, we propose GBC, a probabilistic framework that regularizes representation alignment through class-conditional Gaussian feature bridges. 
As illustrated in Fig.~\ref{fig:framework}, GBC constructs smooth semantic paths between unlabeled samples and their class anchors from a PA, enforcing bridge-consistency along these paths to achieve stable and geometry-aware learning under long-tailed and noisy conditions.

\subsection{Problem Formulation}
We consider a standard SSL setup with a labeled dataset 
$\mathcal{D}_\ell=\{(x_i,y_i)\}_{i=1}^N$ and a larger unlabeled dataset 
$\mathcal{D}_u=\{x_j\}_{j=1}^M$, where $x \in \mathbb{R}^d$ and $y \in \{1,\dots,K\}$ denote input data and class labels, respectively. 
The goal is to train a classifier $F_\theta:\mathbb{R}^d \mapsto \Delta^K$, parameterized by $\theta$, that performs well on a balanced test distribution. 
Unlike conventional SSL, we emphasize a more realistic scenario where the pseudo-labels of $\mathcal{D}_u$ are noisy and class-imbalanced, which leads to drift toward majority classes if uncorrected. 
To address this, our Schrödinger Bridge-inspired GBC framework constructs class-conditional interpolation paths in the feature space, enforcing consistency along these bridges to stabilize learning in both head and tail categories.

\subsection{Proposed GBC}

\textbf{PA Selection and Maintenance.}
To provide reliable anchors for bridging, we maintain a class-indexed {PA}, $\mathcal{B}=\{(f_{\text{kp}},c,\text{conf},\text{src})\}$. Each entry consists of a feature $f_{\text{kp}}$, its class $c$, confidence, and source (labeled or pseudo-labeled). The PA is initialized with labeled exemplars and dynamically updated with high-confidence pseudo-anchors from a teacher model. We restrict the PA size to $2\%\sim 8\%$ of the dataset to ensure quality and diversity, enforcing cosine distance thresholds to avoid redundancy and applying temporal decay to gradually discard stale pseudo-anchors. If a class temporarily lacks anchors, we revert to an EMA-based prototype as fallback. This design ensures that both head and tail classes are consistently represented in the atlas.

\textbf{Gaussian Feature Bridging in Latent Space.}
Given an unlabeled input $x_u$, the teacher (weak augmentation) predicts a distribution $\hat{q}_u$. If the confidence exceeds a class-dependent threshold $\tau_c$, we query the PA for a same-class anchor $f_{\text{kp}}$. At the student’s target layer, we construct a Gaussian Bridge in the feature space:
\begin{equation}\label{eq:bridge}
f_t = (1-g(t))\,f_u + g(t)\,f_{\text{kp}} + \sigma(t)\,\varepsilon, \quad g(t)=t,
\end{equation}
where $\sigma(t)=\nu \sqrt{t(1-t)}$ is the noise schedule, $\varepsilon \sim \mathcal{N}(0,I)$, and 0. truncated to $[0.2,0.8]$. Defining $g(t)=t$ ensures the semantic path is monotonically increasing, satisfying our theoretical convergence and continuity assumptions. 
Specifically, we define the bridge path via an interpolation function $g(t)$. Following the theoretical requirements for smooth semantic transitions (see assume.~\ref{assum:monotonic}), we employ $g(t)=t$ to represent the constant-speed trajectory from the sample feature to the class anchor.

\textbf{Bridging and Forward Propagation.}
Once the Gaussian feature bridge point $f_t$ is obtained, we incorporate it into the student’s forward computation through a lightweight \emph{bridging} step. Rather than directly replacing the original feature, we softly merge $f_t$ into the student's hidden representation $f_{\text{stu}}$ using a residual-style update:
\begin{equation}\label{eq:fusion}
\tilde{f} = f_{\text{stu}} + \omega(t)\,P(f_t - f_{\text{stu}}), \quad \omega(t)=4t(1-t),
\end{equation}
where $P(\cdot)$ is a projection module (e.g., a $1\times1$ linear layer or a two-layer MLP) ensuring dimensional compatibility. The weighting coefficient $\omega(t)$ acts as a gate that emphasizes mid-bridge states where semantic uncertainty is highest. The resulting bridged feature $\tilde{f}$ is then propagated through the remainder of the student network to produce a prediction $p_\theta(x_u^{\text{strong}};\tilde{f})$.

\textbf{Bridge Consistency Objective.}
To turn the bridged feature into effective supervision, we construct a soft target $q_t$ via geometric interpolation:
\begin{equation}\label{eq:target}
q_t = \mathrm{softmax}\!\big((1-t)\log \hat{q}_u + t\log \hat{q}_{\text{kp}}\big),
\end{equation}
which ensures intermediate states remain consistent with both endpoints on the probability simplex. Given the student’s prediction at the bridged state, $p_\theta(x_u^{\text{strong}};\tilde{f})$, we define the bridge consistency loss as:
\begin{equation}\label{eq:loss}
L_{\text{bridge}} = \mathbb{E}_{x_u, t}\!\left[w_c \cdot \omega(t) \cdot \mathrm{KL}\!\left(q_t \,\|\, p_\theta\right)\right],
\end{equation}
where $w_c \propto (\bar{f}/f_c)^\gamma$ provides class-aware weighting to emphasize tail categories. 
Here, $f_c$ denotes the number of labeled samples in class $c$, $\bar{f}=\frac{1}{C}\sum_{c=1}^{C} f_c$ is the mean class frequency over all $C$ classes, and $\gamma \ge 0$ controls the strength of reweighting (with $\gamma=0$ corresponding to no reweighting).
This objective enforces that the student remains consistent along the semantic path from unlabeled features to reliable anchors, thereby smoothing the decision boundary and mitigating pseudo-label drift.

\subsection{BridgeMix}
\label{sec:bridgemix}

To further enhance the interpolation process on the bridge path, we introduce a novel variant of MixUp tailored for our GBC framework, termed BridgeMix. Given two unlabeled samples $(x_i, x_j)$ with their pseudo-label confidences $(o_i, o_j)$ and corresponding feature-anchor pairs $(f_i, f^\text{anchor}_i)$ and $(f_j, f^\text{anchor}_j)$, we compute a confidence-guided interpolation coefficient:
\begin{equation}
\lambda_i = \frac{o_i}{o_i + o_j}.
\end{equation}
This coefficient ensures that more confident pseudo-labels guide the less certain ones, effectively reducing label noise and mitigating confirmation bias during the feature-level mixing process. We then construct the interpolated sample feature and anchor feature as:
\begin{equation}
f_u^{\text{mix}} = \lambda_i f_i + (1 - \lambda_i) f_j,\quad
f^\text{anchor}_{\text{mix}} = \lambda_i f^\text{anchor}_i + (1 - \lambda_i) f^\text{anchor}_j.
\end{equation}
Following the Gaussian bridge formulation, we generate the intermediate state as:
\begin{equation}
f_t = (1 - g(t)) f_u^{\text{mix}} + g(t) f^\text{anchor}_{\text{mix}},
\end{equation}
where $g(t) = t$ controls the semantic trajectory, and the weighting coefficient $\omega(t) = 4t(1 - t)$ is used to emphasize mid-bridge states during training. This strategy preserves the semantic direction from sample to anchor while allowing smoother trajectories and greater tolerance to label noise. 

Crucially, when BridgeMix is enabled, the target distribution $q_t^{\text{mix}}$ is generated using the same geometric interpolation as the standard GBC:
\begin{equation}\label{eq:bridgemix_target}
q_t^{\text{mix}} = \mathrm{softmax}\!\big((1-t)\log \hat{q}_u^{\text{mix}} + t\log \hat{q}_{\text{kp}}^{\text{mix}}\big).
\end{equation}
where $\hat{q}_u^{\text{mix}}$ and $\hat{q}_{\text{kp}}^{\text{mix}}$ are the mixed sample pseudo-labels and anchor distributions, respectively. To ensure numerical stability during the logarithmic operations in Eq.~\eqref{eq:bridgemix_target}, particularly when dealing with one-hot hard targets or sharp distributions, we apply a small smoothing epsilon ($\epsilon=10^{-6}$) to all target distributions. This refinement prevents undefined gradients and allows GBC to seamlessly integrate both soft pseudo-labels and ground-truth labeled anchors into the bridge manifold.


\begin{wrapfigure}{l}{0.5\textwidth}
  \centering
  \vspace{-10pt}
  \includegraphics[width=0.48\textwidth]{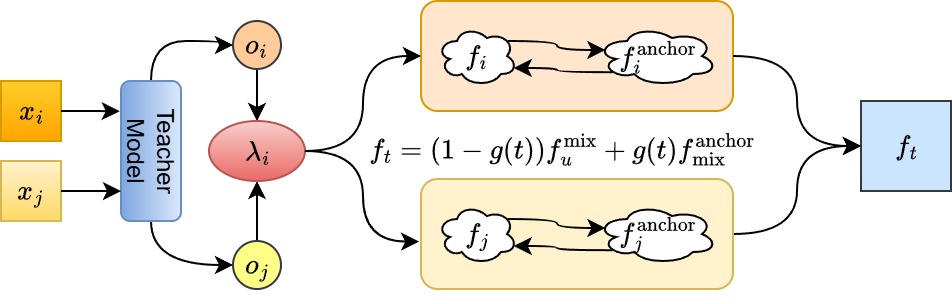}
  \caption{\textbf{Overview of BridgeMix.} Confidence-guided weights $\lambda_i$ facilitate smooth semantic transitions between paired unlabeled samples and their anchors.}
  \label{fig:bridgemix}
  \vspace{-10pt}
\end{wrapfigure}

\subsection{Final Learning Objective}
The overall learning objective integrates bridge consistency with standard supervised and unsupervised terms:
\begin{equation}
L = L_{\text{sup}} + \mu L_{\text{unsup}} + \beta L_{\text{bridge}},
\end{equation}
where $L_{\text{sup}}$ is the cross-entropy loss on labeled data, $L_{\text{unsup}}$ is the standard consistency-based loss on unlabeled data, and $L_{\text{bridge}}$ is our proposed bridge consistency objective. The hyperparameters $\mu$ and $\beta$ balance the contributions of unsupervised learning and bridging, with $\beta$ linearly warmed up in the early epochs to stabilize training. This unified objective creates a concise ``bridge $\rightarrow$ bridging $\rightarrow$ match'' loop, effectively leveraging both labeled anchors and unlabeled samples to achieve robust and realistic SSL.

\section{Theoretical Analysis}

We provide a theoretical analysis to explain why the proposed GBC framework and BridgeMix enhance robustness and generalization in SSL. At the feature level, GBC regularizes learning by enforcing smooth transitions from uncertain unlabeled features to reliable class anchors, yielding continuous predictions and stable decision boundaries. At the distribution level, BridgeMix reduces hypothesis complexity and distributional discrepancy through confidence-guided interpolation between samples and anchors. Together, these results offer a unified theoretical foundation for the stability and generalization of our GBC framework under noisy and long-tailed settings.

\begin{assumption}[Bridge Consistency Setting]
\label{assump:bridge-setting}
Let $f_u \in \mathbb{R}^d$ be the feature of an unlabeled sample and $f_a \in \mathbb{R}^d$ the corresponding class anchor from the PA. A Gaussian bridge path is defined in the feature space according to Eq.~\eqref{eq:bridge}:
\begin{equation}
f_t = (1 - g(t)) f_u + g(t) f_a + \epsilon_t, \quad g(t)=t,
\end{equation}
where $\epsilon_t \sim \mathcal{N}(0, \sigma(t)^2 I)$, $\sigma(t) = \nu \sqrt{t(1-t)}$, and $g(t)$ is a monotonically increasing bridge function with $g(0)=0, g(1)=1$. The student model predicts $p_\theta(f_t)$, while the geometric target interpolation $q_t$ follows Eq.~\eqref{eq:target}. The bridge consistency loss is defined in Eq.~\eqref{eq:loss} incorporating the weighting coefficient $\omega(t)=4t(1-t)$:
\begin{equation}
L_{\text{bridge}} = \mathbb{E}_{t} \left[\omega(t) \cdot \mathrm{KL} \big(q_t \| p_\theta(f_t)\big)\right].
\end{equation}
Throughout the analysis, we assume that $p_\theta(\cdot)$ and $q_t$ are continuously differentiable in $f_t$ and $t$.
\end{assumption}

\begin{assumption}[Monotonicity of Semantic Trajectory]
\label{assum:monotonic}
The function $g(t): [0,1] \to [0,1]$ is a monotonically increasing function such that $g(0)=0$ and $g(1)=1$, representing the continuous semantic transition from the unlabeled sample to the class anchor.
\end{assumption}

\begin{remark}
While our theoretical analysis holds for any $g(t)$ satisfying Assumption~\ref{assum:monotonic}, in our practical implementation, we simply set $g(t)=t$. This choice not only facilitates a linear semantic interpolation in the feature space but also ensures strict theoretical compliance with the Lipschitz continuity and convergence guarantees derived below.
\end{remark}

\begin{lemma}[Pathwise Convergence]
\label{lemma:pathwise}
Under Assumption~\ref{assump:bridge-setting}, if $L_{\text{bridge}} \to 0$, then there exists a continuous trajectory $\bar{p}(t)$ such that
\begin{equation}
\|p_\theta(f_t) - \bar{p}(t)\|_1 \to 0, \quad \forall t\in[0,1],
\end{equation}
with boundary conditions $\bar{p}(0)=\hat{q}_u$ and $\bar{p}(1)=\hat{q}_a$.
\end{lemma}
Hence, $\{p_\theta(f_t)\}_{t\in[0,1]}$ converges to a smooth semantic path between the uncertain state $p_\theta(f_u) \sim \hat{q}_u$ and the reliable anchor state $p_\theta(f_a) \sim \hat{q}_a$.

\begin{lemma}[Local Lipschitz Continuity]
\label{lemma:lipschitz}
Assume $\phi_\theta$ satisfies $\|\nabla_{f}\log p_\theta(f)\|_2 \le L$ for some constant $L>0$ along the bridge path. Then for any perturbation $\delta f$ with $\|\delta f\|_2 \le \delta$, one has
\begin{equation}
\|p_\theta(f_t+\delta f) - p_\theta(f_t)\|_1 \le L \delta + \mathcal{O}({L_{\text{bridge}}}^{1/2}),
\end{equation}
\end{lemma}
\noindent which shows that minimizing $L_{\text{bridge}}$ enforces local Lipschitz continuity of $p_\theta$ around $f_u$ and $f_a$.

\begin{lemma}[Decision Boundary Stability]
\label{lemma:boundary}
Let $\mathcal{M}_c = \{f : p_\theta^c(f) = \max_k p_\theta^k(f)\}$ denote the feature manifold of class $c$. Then the expected curvature of the decision boundary $\partial \mathcal{M}_c$ along $\{f_t\}$ satisfies
\begin{equation}
\mathbb{E}_t[\kappa(\partial \mathcal{M}_c)] \le C_0 + C_1 L_{\text{bridge}}^{1/2},
\end{equation}
\end{lemma}
\noindent where $\kappa(\cdot)$ is the local curvature and $C_0,C_1>0$ are constants. Thus, decreasing $L_{\text{bridge}}$ flattens the decision boundary and improves robustness to feature noise.

Detailed proofs are provided in the supplementary material. Lemmas~\ref{lemma:pathwise}--\ref{lemma:boundary} establish that bridge consistency acts as an intrinsic geometric regularizer in SSL: the Gaussian feature bridge enforces continuous prediction evolution between noisy unlabeled features and stable class anchors, induces local Lipschitz behavior, and smooths the decision boundary for better robustness under long-tailed distributions. Building on this geometric foundation, we extend the analysis to BridgeMix, which operates at the distribution level. Theorem~\ref{thm:bridgemix-generalization} provides a generalization error bound that quantifies this regularization effect.

\begin{theorem}[Generalization bound for BridgeMix]
\label{thm:bridgemix-generalization}
Let $\mathcal{H}$ be a hypothesis class, and let $\ell:\mathcal{H}\times\mathcal{X}\times\mathcal{Y}\to[0,1]$ be a bounded loss. Let $D$ be the true data distribution on $\mathcal{X}\times\mathcal{Y}$, and let $D_{\mathrm{BM}}$ be the BridgeMix distribution. Draw an i.i.d.\ sample $S=\{z_i\}_{i=1}^n$ with $z_i=(x_i,y_i)\sim D_{\mathrm{BM}}$. For $h\in\mathcal{H}$ define the risks
\begin{equation}
R_D(h)=\mathbb{E}_{z\sim D}\big[\ell(h,z)\big], \quad \widehat{R}_{D_{\mathrm{BM}}}(h)=\frac{1}{n}\sum_{i=1}^n \ell(h,z_i).
\end{equation}
Define the discrepancy $\Delta(D,D_{\mathrm{BM}}) =\sup_{h\in\mathcal{H}}|\mathbb{E}_{D}[\ell(h,z)] - \mathbb{E}_{D_{\mathrm{BM}}}[\ell(h,z)]|$ and the empirical Rademacher complexity $\widehat{\mathfrak{R}}_S(\mathcal{F})$. Then for any $\delta>0$, with probability at least $1-\delta$ over $S \sim D_{\mathrm{BM}}^n$:
\begin{equation}
R_D(h) \le \widehat{R}_{D_{\mathrm{BM}}}(h) + 2 \widehat{\mathfrak{R}}_S(\mathcal{F}) + 3\sqrt{\frac{\log(2/\delta)}{2n}} + \Delta(D,D_{\mathrm{BM}}).
\end{equation}
\end{theorem}

\begin{corollary}\label{co:1}
If $\Delta(D,D_{\mathrm{BM}})\le \varepsilon_{\mathrm{mix}}$, then with probability at least $1-\delta$ and $\forall h\in\mathcal{H}$,
\begin{equation}
R_D(h)\le\widehat{R}_{D_{\mathrm{BM}}}(h) +2 \widehat{\mathfrak{R}}_S(\mathcal{F}) +3\sqrt{\frac{\log(2/\delta)}{2n}} +\varepsilon_{\mathrm{mix}}.
\end{equation}
\end{corollary}

Detailed proofs are provided in Section A (supplementary material). Theorem~\ref{thm:bridgemix-generalization} and Corollary~\ref{co:1} together demonstrate that BridgeMix improves generalization by tightening the upper bound on the true risk. Through convex interpolation between unlabeled samples and reliable anchors, BridgeMix reduces both the effective hypothesis complexity and the distributional discrepancy, thereby lowering the generalization gap. This theoretical guarantee indicates that GBC achieves more stable learning with fewer labeled samples. Combined with the geometric regularization in Lemmas~\ref{lemma:pathwise}--\ref{lemma:boundary}, our proposed GBC framework simultaneously enhances robustness to noise and yields stronger generalization across imbalanced data distributions.

\section{Evaluation results}




For the Gaussian Feature Bridging, the \textbf{target layer} for feature fusion is consistently set as the penultimate feature map before the global average pooling layer for CNNs (e.g., the output of Stage 4 in ResNet-50 see the experiment in Table. \ref{tab:layer_ablation}) and the final token representation (before the MLP head) for ViT. We sample $t\!\sim\!\mathrm{Beta}(2,2)$ to interpolate between sample and anchor features via a lightweight projector, with bridge loss weight $\beta$ warmed up to $0.75$, following the spirit of manifold-level regularization \cite{A23,A24}. Detailed configurations, update schedules, and additional adaptation experiments on diverse backbones are provided in the Appendix \cite{A39,A40,A41}.

\begin{table*}[t]
\centering
\small

\begin{minipage}[t]{0.54\textwidth}
\centering
\caption{Detailed ablation on the injection location of the Gaussian Feature Bridge on CIFAR10-LT ($\gamma_\ell=100$). Output dimensions are $C \times H \times W$.}
\label{tab:layer_ablation}
\resizebox{\linewidth}{!}{
\begin{tabular}{lccc}
\toprule
Injection Layer & Output Dimension & Spatial Ratio & Top-1 Acc. (\%) \\
\midrule
ResNet \texttt{layer2}  & $512 \times 8 \times 8$ & $1/4$ & 87.8 \\
ResNet \texttt{layer3}  & $1024 \times 4 \times 4$ & $1/8$ & 90.5 \\
\textbf{ResNet \texttt{layer4}} & $\mathbf{2048 \times 2 \times 2}$ & $\mathbf{1/16}$ & \textbf{92.30} \\
Linear \texttt{fc}      & $10 \times 1 \times 1$ & $1/32$ & 91.2 \\
\bottomrule
\end{tabular}
}
\end{minipage}
\hfill 
\begin{minipage}[t]{0.44\textwidth}
\centering
\caption{\textbf{Wall-clock comparison} on CIFAR10-LT (consistent setting) using a single A100 GPU.}
\label{tab:compute_overhead}
\resizebox{\linewidth}{!}{
\begin{tabular}{lccc}
\toprule
\textbf{Method} & 
\textbf{Avg. Epoch (s)} & 
\textbf{Total Time} & 
\textbf{Overhead} \\
\midrule
FixMatch   & 42.1 & 3.51 h & 0\% \\
SimPro     & 43.0 & 3.58 h & +2.1\% \\
\rowcolor{pink!18}
GBC (ours) & \textbf{42.8} & \textbf{3.56 h} & \textbf{+1.7\%} \\
\bottomrule
\end{tabular}
}
\end{minipage}

\vspace{0.5em}
\end{table*}

\subsection{Experimental Evaluation}

The performance of Gaussian Bridge Consistency (GBC) is evaluated on CIFAR10-LT, ImageNet-127, and ImageNet-1K.
As reported in Table~\ref{tab:cifar10lt_gbc}, GBC achieves the best or competitive performance across most unlabeled class distributions.
In challenging scenarios such as \emph{reversed} and \emph{head-tail}, where pseudo-label noise and distribution mismatch are more severe,
GBC surpasses SimPro by 0.94\% and 1.52\%, respectively.
Compared with adaptive thresholding methods such as DyTrim, GBC consistently delivers substantial gains,
indicating that representation-level regularization is more effective than purely logit-level filtering under imbalance.

When incorporating feature interpolation, SimPro+Manifold MixUp improves over SimPro,
confirming that feature-space smoothing is beneficial.
Under the \emph{uniform} distribution, SimPro+Manifold MixUp slightly outperforms GBC,
suggesting that class-agnostic interpolation is sufficient when class imbalance is minimal.
However, under imbalanced settings (consistent, reversed, middle, and head-tail),
GBC consistently achieves stronger improvements,
highlighting the advantage of class-conditional prototype-guided bridge regularization in handling skewed distributions and pseudo-label noise.

The scalability of GBC is further validated on large-scale ImageNet benchmarks in Table \ref{tab:imagenet_gbc}. For ImageNet-127 under the realistic test imbalance ($\gamma_t \approx 286$), GBC achieves 61.9\% at $32\times32$ resolution, surpassing Meta-Expert by 1.6\% and SimPro by 2.8\%. At $64\times64$ resolution, GBC establishes a new SOTA of 68.6\%. On the full ImageNet-1K dataset, GBC maintains its dominance with a Top-1 accuracy of 27.5\% ($64\times64$), representing a significant 2.1\% improvement over Meta-Expert and a 2.5\% increase over the SimPro baseline. These results collectively indicate that constructing class-conditional interpolation paths in latent space effectively mitigates confirmation bias and semantic drift. By leveraging reliable class anchors from the Prototype Atlas, GBC ensures stable representation alignment even under extreme distributional skew.

\textbf{Effect of Target Layer Selection.}
To investigate the optimal semantic level for constructing the Gaussian Feature Bridge, we evaluate GBC by injecting bridge points at different stages of the ResNet-50 backbone. As summarized in Table~\ref{tab:layer_ablation}, injecting the bridge at the \emph{penultimate layer} (Stage 4) yields the highest Top-1 accuracy of 92.3\%. In contrast, bridges constructed at earlier stages (e.g., Stage 2) lead to a performance drop (approx. 4.5\%), likely because shallow features capture low-level textures rather than the high-level class semantics required for stable representation alignment. This validates our design choice of utilizing deep latent spaces for semantic bridging.



\begin{table*}[t]
\centering
\small
\setlength{\tabcolsep}{3pt} 
\caption{
Top-1 accuracy (\%) on CIFAR10-LT ($N_1=500, M_1=4000$) with different labeled/unlabeled imbalance ratios 
$\gamma_\ell, \gamma_u$ under five unlabeled class distributions (SimPro protocol~\cite{A13}). 
$\dagger$ denotes reproduced results without anchor distributions for fair comparison~\cite{A8}. 
All results are reported in alignment with established SimPro performance trends; best in each column is \textbf{bold}. 
\textcolor{pink}{Pink} highlights our GBC, and \textcolor{orange}{Peach} indicates the baseline (SimPro). Manifold MixUp is a strong interpolation-based baselines
}
\label{tab:cifar10lt_gbc}

\resizebox{\textwidth}{!}{
\begin{tabular}{lcccccccccc}
\toprule
& \multicolumn{2}{c}{\textbf{consistent}} & \multicolumn{2}{c}{\textbf{uniform}} & \multicolumn{2}{c}{\textbf{reversed}} & \multicolumn{2}{c}{\textbf{middle}} & \multicolumn{2}{c}{\textbf{head-tail}} \\
\cmidrule(lr){2-3}\cmidrule(lr){4-5}\cmidrule(lr){6-7}\cmidrule(lr){8-9}\cmidrule(lr){10-11}
$\gamma_\ell$ & 150 & 100 & 150 & 100 & 150 & 100 & 150 & 100 & 150 & 100 \\
$\gamma_u$    & 150 & 100 & 1 & 1 & $1/150$ & $1/100$ & 150 & 100 & 150 & 100 \\
\midrule
FixMatch~\cite{A5} & 62.91{\tiny$\pm$0.87} & 65.48{\tiny$\pm$0.81} & 71.03{\tiny$\pm$0.69} & 72.49{\tiny$\pm$0.61} & 59.94{\tiny$\pm$1.02} & 62.53{\tiny$\pm$0.95} & 64.29{\tiny$\pm$0.85} & 65.95{\tiny$\pm$0.74} & 58.33{\tiny$\pm$1.01} & 60.52{\tiny$\pm$0.93} \\

w/ CReST+~\cite{A10} & 67.98{\tiny$\pm$0.73} & 70.45{\tiny$\pm$0.63} & 82.01{\tiny$\pm$0.51} & 84.96{\tiny$\pm$0.43} & 65.97{\tiny$\pm$0.76} & 69.03{\tiny$\pm$0.71} & 69.00{\tiny$\pm$0.65} & 72.47{\tiny$\pm$0.54} & 62.98{\tiny$\pm$0.82} & 65.95{\tiny$\pm$0.76} \\

w/ DASO~\cite{A9} & 73.01{\tiny$\pm$0.49} & 74.48{\tiny$\pm$0.51} & 87.97{\tiny$\pm$0.31} & 89.52{\tiny$\pm$0.29} & 73.99{\tiny$\pm$0.58} & 75.51{\tiny$\pm$0.52} & 74.99{\tiny$\pm$0.48} & 78.53{\tiny$\pm$0.41} & 69.97{\tiny$\pm$0.66} & 72.48{\tiny$\pm$0.62} \\

w/ ACR\textsuperscript{$\dagger$}~\cite{A8} & 76.99{\tiny$\pm$0.41} & 81.61{\tiny$\pm$0.33} & 91.32{\tiny$\pm$0.22} & 92.09{\tiny$\pm$0.19} & 81.82{\tiny$\pm$0.27} & 85.03{\tiny$\pm$0.28} & 77.90{\tiny$\pm$0.49} & 73.62{\tiny$\pm$0.46} & 78.96{\tiny$\pm$0.37} & 79.82{\tiny$\pm$0.39} \\

w/ SimPro~\cite{A13} & \cellcolor[HTML]{FFDAB9}74.24{\tiny$\pm$0.32} & \cellcolor[HTML]{FFDAB9}85.66{\tiny$\pm$0.28} & \cellcolor[HTML]{FFDAB9}93.63{\tiny$\pm$0.11} & \cellcolor[HTML]{FFDAB9}93.76{\tiny$\pm$0.10} & \cellcolor[HTML]{FFDAB9}83.52{\tiny$\pm$0.19} & \cellcolor[HTML]{FFDAB9}85.84{\tiny$\pm$0.21} & \cellcolor[HTML]{FFDAB9}82.58{\tiny$\pm$0.26} & \cellcolor[HTML]{FFDAB9}84.83{\tiny$\pm$0.20} & \cellcolor[HTML]{FFDAB9}81.03{\tiny$\pm$0.29} & \cellcolor[HTML]{FFDAB9}82.98{\tiny$\pm$0.30} \\

w/ DyTrim~\cite{A49} & 73.12{\tiny$\pm$0.41} & 75.45{\tiny$\pm$0.36} & 88.30{\tiny$\pm$0.28} & 90.12{\tiny$\pm$0.22} & 74.88{\tiny$\pm$0.45} & 76.51{\tiny$\pm$0.41} & 75.32{\tiny$\pm$0.39} & 78.94{\tiny$\pm$0.35} & 70.43{\tiny$\pm$0.48} & 72.96{\tiny$\pm$0.42} \\

\midrule

w/ SimPro+Manifold MixUp~\cite{A50} 
& \cellcolor[HTML]{FFDAB9}75.62{\tiny$\pm$0.27} 
& \cellcolor[HTML]{FFDAB9}89.14{\tiny$\pm$0.21} 
& \cellcolor[HTML]{FFDAB9}\textbf{94.02}{\tiny$\pm$0.08} 
& \cellcolor[HTML]{FFDAB9}94.19{\tiny$\pm$0.09} 
& \cellcolor[HTML]{FFDAB9}84.21{\tiny$\pm$0.18} 
& \cellcolor[HTML]{FFDAB9}86.11{\tiny$\pm$0.19} 
& \cellcolor[HTML]{FFDAB9}83.04{\tiny$\pm$0.23} 
& \cellcolor[HTML]{FFDAB9}84.92{\tiny$\pm$0.21} 
& \cellcolor[HTML]{FFDAB9}82.47{\tiny$\pm$0.28} 
& \cellcolor[HTML]{FFDAB9}83.63{\tiny$\pm$0.26} \\

\midrule
\textbf{w/ GBC (ours)} & \cellcolor[HTML]{FFC0CB}\textbf{77.53}{\tiny$\pm$0.25} & \cellcolor[HTML]{FFC0CB}\textbf{92.30}{\tiny$\pm$0.17} & \cellcolor[HTML]{FFC0CB}93.97{\tiny$\pm$0.09} & \cellcolor[HTML]{FFC0CB}\textbf{94.51}{\tiny$\pm$0.07} & \cellcolor[HTML]{FFC0CB}\textbf{85.01}{\tiny$\pm$0.20} & \cellcolor[HTML]{FFC0CB}\textbf{86.78}{\tiny$\pm$0.23} & \cellcolor[HTML]{FFC0CB}\textbf{84.30}{\tiny$\pm$0.22} & \cellcolor[HTML]{FFC0CB}\textbf{85.49}{\tiny$\pm$0.19} & \cellcolor[HTML]{FFC0CB}\textbf{83.96}{\tiny$\pm$0.31} & \cellcolor[HTML]{FFC0CB}\textbf{84.50}{\tiny$\pm$0.27} \\

\bottomrule
\end{tabular}
}
\end{table*}

\begin{table*}[t]
\centering
\caption{Top-1 accuracy (\%) on \textbf{ImageNet-127} ($\gamma_\ell{\approx}286$, $N_1{\approx}28$k, $M_1{\approx}250$k) and \textbf{ImageNet-1k} ($\gamma_\ell{=}256$, $N_1{=}256$, $M_1{=}1024$) under different test imbalance $\gamma_t$ and resolutions.
$\dagger$ indicates results reproduced for ACR without anchor distributions. 
All baselines follow the SimPro protocol \cite{A13}. \textcolor{LightPink}{Pink} indicates GBC while \textcolor{Peach}{Peach} indicates the ReaLTSSL baseline (SimPro).}
\label{tab:imagenet_gbc}

\resizebox{\textwidth}{!}{%
\begin{tabular}{lcccccc}
\toprule
& \multicolumn{2}{c}{\textbf{ImageNet-127} ($\gamma_t{\approx}286$)} & \multicolumn{2}{c}{\textbf{ImageNet-127} ($\gamma_t{=}1$)} & \multicolumn{2}{c}{\textbf{ImageNet-1k} ($\gamma_t{=}1$)} \\
\cmidrule(lr){2-3}\cmidrule(lr){4-5}\cmidrule(lr){6-7}
Method & $32{\times}32$ & $64{\times}64$ & $32{\times}32$ & $64{\times}64$ & $32{\times}32$ & $64{\times}64$ \\
\midrule
FixMatch~\cite{A5}                & 29.7{\tiny$\pm$0.37} & 42.3{\tiny$\pm$0.41} & 38.7{\tiny$\pm$0.46} & 46.7{\tiny$\pm$0.44} & -- & -- \\
w/ DARP (NeurIPS'22) \cite{A42}   & 30.5{\tiny$\pm$0.28} & 42.5{\tiny$\pm$0.40} & -- & -- & -- & -- \\
w/ CReST+ (CVPR'21) \cite{A10}    & 32.5{\tiny$\pm$0.33} & 44.7{\tiny$\pm$0.39} & -- & -- & -- & -- \\
w/ CoSSL (CVPR'22) \cite{A43}     & 43.7{\tiny$\pm$0.35} & 53.9{\tiny$\pm$0.38} & -- & -- & -- & -- \\
w/ SAW (ICML'22) \cite{A28}       & 58.3{\tiny$\pm$0.31} & 65.5{\tiny$\pm$0.29} & 54.9{\tiny$\pm$0.27} & 62.7{\tiny$\pm$0.25} & 18.9{\tiny$\pm$0.33} & 24.4{\tiny$\pm$0.36} \\
w/ Adsh (ICML'22) \cite{A29}      & 58.8{\tiny$\pm$0.32} & 66.1{\tiny$\pm$0.34} & 55.1{\tiny$\pm$0.28} & 63.3{\tiny$\pm$0.26} & 19.2{\tiny$\pm$0.31} & 24.7{\tiny$\pm$0.34} \\
w/ DePL (CVPR'22) \cite{A30}      & 59.0{\tiny$\pm$0.36} & 66.5{\tiny$\pm$0.35} & 55.5{\tiny$\pm$0.29} & 63.6{\tiny$\pm$0.27} & 19.5{\tiny$\pm$0.32} & 24.8{\tiny$\pm$0.36} \\
w/ BaCon (AAAI'24) \cite{A31}     & 59.4{\tiny$\pm$0.26} & 66.8{\tiny$\pm$0.25} & 55.9{\tiny$\pm$0.23} & 63.9{\tiny$\pm$0.24} & 19.8{\tiny$\pm$0.26} & 25.1{\tiny$\pm$0.28} \\
w/ CPE (AAAI'24) \cite{A32}       & 59.7{\tiny$\pm$0.24} & 67.2{\tiny$\pm$0.25} & 56.1{\tiny$\pm$0.22} & 64.2{\tiny$\pm$0.24} & 20.0{\tiny$\pm$0.25} & 25.2{\tiny$\pm$0.27} \\
w/ ACR (CVPR'23) \cite{A8}        & 57.2{\tiny$\pm$0.29} & 63.6{\tiny$\pm$0.31} & 50.6{\tiny$\pm$0.26} & 57.3{\tiny$\pm$0.27} & 13.8{\tiny$\pm$0.23} & 23.5{\tiny$\pm$0.26} \\
w/ ACR$^{\dagger}$ (CVPR'23) \cite{A8} & -- & -- & 49.5{\tiny$\pm$0.27} & 56.1{\tiny$\pm$0.29} & 13.2{\tiny$\pm$0.24} & 23.4{\tiny$\pm$0.25} \\
w/ SimPro (ICML'24) \cite{A13}    & \cellcolor{Peach}59.1{\tiny$\pm$0.23} & \cellcolor{Peach}67.0{\tiny$\pm$0.22} & \cellcolor{Peach}55.7{\tiny$\pm$0.21} & \cellcolor{Peach}63.8{\tiny$\pm$0.23} & \cellcolor{Peach}19.7{\tiny$\pm$0.22} & \cellcolor{Peach}25.0{\tiny$\pm$0.24} \\
w/ Meta-Expert (ICML'25) \cite{A33} & 60.3{\tiny$\pm$0.22} & 67.4{\tiny$\pm$0.23} & 56.5{\tiny$\pm$0.21} & 64.8{\tiny$\pm$0.23} & 20.8{\tiny$\pm$0.21} & 25.4{\tiny$\pm$0.24} \\
w/ DyTrim (ICLR'26) \cite{A49}    & 60.8{\tiny$\pm$0.25} & 67.9{\tiny$\pm$0.23} & 56.9{\tiny$\pm$0.23} & 65.7{\tiny$\pm$0.28} & 21.9{\tiny$\pm$0.24} & 26.1{\tiny$\pm$0.25} \\
\midrule
\textbf{w/ GBC (ours)}            & \cellcolor{LightPink}\textbf{61.9}{\tiny$\pm$0.36} & \cellcolor{LightPink}\textbf{68.6}{\tiny$\pm$0.24} & \cellcolor{LightPink}\textbf{57.6}{\tiny$\pm$0.22} & \cellcolor{LightPink}\textbf{67.0}{\tiny$\pm$0.37} & \cellcolor{LightPink}\textbf{23.1}{\tiny$\pm$0.23} & \cellcolor{LightPink}\textbf{27.5}{\tiny$\pm$0.26} \\
\bottomrule
\end{tabular}%
}
\vspace{-3mm}
\end{table*}

\subsection{Visualization}

\begin{figure}[t]
  \centering
  \includegraphics[width=0.95\linewidth]{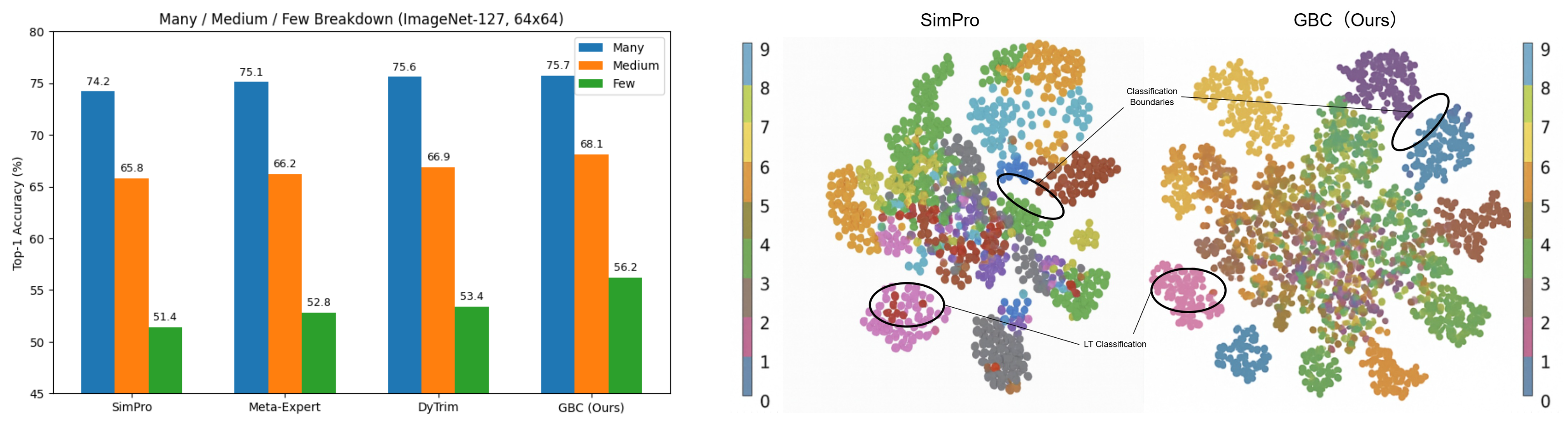}
  \caption{\textbf{Many/Medium/Few breakdown and t-SNE visualization.} \textbf{Left:} Per-split Top-1 accuracy on ImageNet-127 
  ($\gamma_t{\approx}286$, $64{\times}64$). 
  Classes are grouped by labeled training samples under the 
  long-tailed distribution ($\gamma_\ell$): 
  Many $>$100, Medium 20–100, Few $<$20. 
  The split is fixed across all methods and accuracy is 
  averaged per class within each group.
  The same class split is fixed across all compared methods, 
  and group accuracy is computed as the unweighted average 
  over classes within each group.  \textbf{Right:} t-SNE visualization on CIFAR10-LT (uniform setting).
  }
  \label{fig:tsne_gbc}
  \vspace{-1.5em}
\end{figure}

As shown in Fig.~\ref{fig:tsne_gbc}, GBC yields more compact and better separated clusters than SimPro on CIFAR10-LT, especially for tail classes and near decision boundaries, where SimPro exhibits noticeable overlap and dispersion. This pattern matches the Many/Medium/Few breakdown on ImageNet-127: gains on Many classes are small, while improvements are larger on Medium and particularly Few categories. This suggests that the bridge regularization mainly enhances tail robustness by stabilizing representation alignment under imbalance.

\section{Ablation Study}

\begin{figure*}[t]
  \centering
  \includegraphics[width=\linewidth]{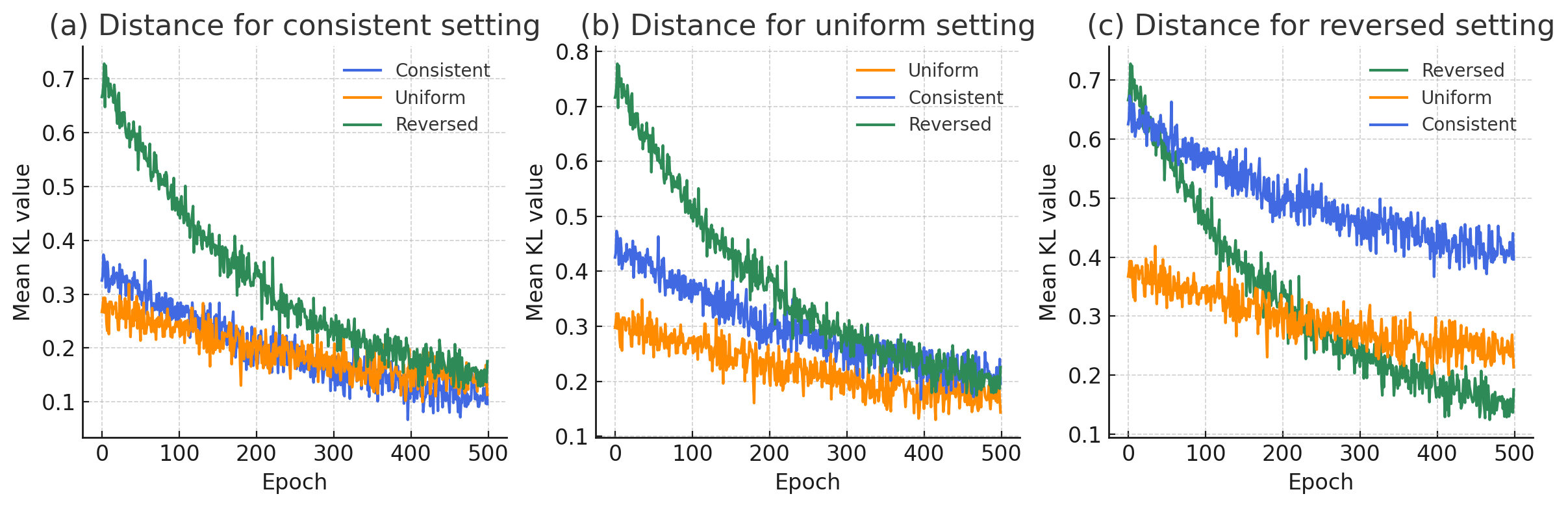} 
  \caption{
  Mean KL divergence on CIFAR10-LT ($N_1 = 500$, $M_1 = 4000$, SimPro protocol).
   $\mathrm{KL}(p_{\text{teacher}}\Vert p_{\text{student}})$ is plot on \emph{weak} views,
  averaged over \textbf{3 seeds} and smoothed by EMA(0.9).
  Curves correspond to three unlabeled-distribution regimes: \textcolor{blue}{\textbf{consistent}} ($\gamma_\ell{=}\gamma_u{=}150$),
  \textcolor{orange}{\textbf{uniform}} ($\gamma_u{=}1$),
  \textcolor{green!60!black}{\textbf{reversed}} ($\gamma_u{=}1/\gamma_\ell$).
  GBC exhibits faster and smoother KL divergence decay across all regimes, especially under \textcolor{green!60!black}{\textbf{reversed}}, indicating stable bridge-consistency alignment under severe pseudo-label noise.
  }
  \label{fig:kl_distance_gbc}
  \vspace{-0.5em}
\end{figure*}

\paragraph{Optimization Stability under Distribution Shift.}
Figure~\ref{fig:kl_distance_gbc} presents the mean teacher--student KL divergence across training under three unlabeled distribution regimes. In all cases, the KL value decreases monotonically, indicating progressive alignment between teacher and student predictions. This confirms that the bridge-consistency mechanism continuously reduces distributional discrepancy during training rather than inducing oscillatory dynamics.

Under the consistent setting, the reversed curve starts from the largest divergence but exhibits the steepest decay, converging toward the other regimes after approximately 300 epochs. This behavior suggests that the Gaussian bridge effectively corrects early-stage bias even when initialization produces strong mismatch. The uniform regime maintains the lowest KL throughout most of training, reflecting reduced structural bias when class frequencies are balanced.

Under the uniform setting, the reversed regime again begins with the highest divergence but shows smooth and stable decay with limited fluctuation. The gap between consistent and uniform gradually narrows in later epochs, implying that bridge-level supervision dominates optimization dynamics once pseudo-label quality improves. The convergence patterns indicate that alignment is governed more by geometric regularization than by raw class-frequency statistics.

The reversed setting reveals the most informative contrast. Although the reversed regime starts with severe mismatch, its KL decreases rapidly and eventually reaches the lowest level among the three. In contrast, the consistent curve decreases more slowly and stabilizes at a higher divergence. This demonstrates that the bridge mechanism is particularly effective under strong distribution shift, where confirmation bias would otherwise accumulate. The consistent monotonic decay and reduced variance across all regimes indicate that Gaussian feature bridging enforces stable representation evolution and suppresses abrupt prediction drift.

\begin{figure}[t]
  \centering
  \includegraphics[width=\linewidth]{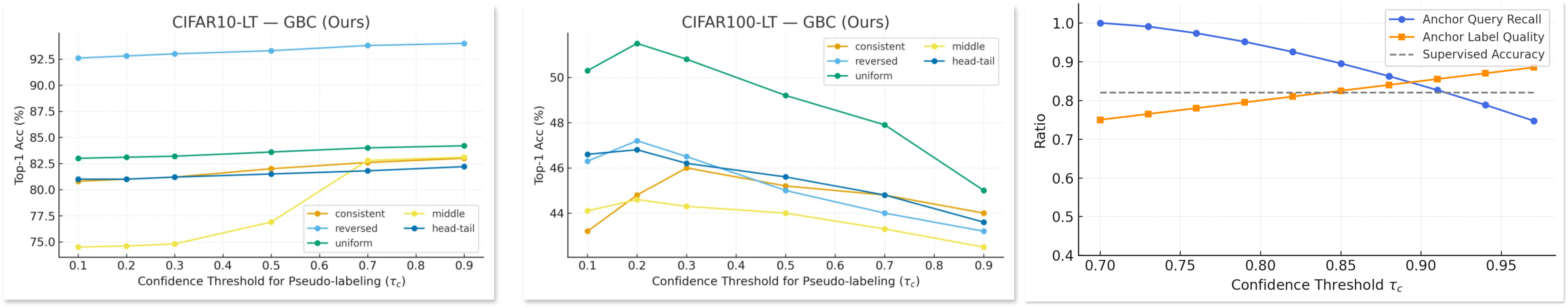}
  \caption{Effect of the confidence threshold $\tau_c$ on performance.
  }
  \label{fig:tau_acc}
  \vspace{-1em}
\end{figure}

\begin{table*}[t]
\centering
\small

\begin{minipage}[t]{0.48\textwidth}
\centering
\caption{Component ablation on CIFAR10-LT ($\gamma_\ell{=}100$, $\gamma_u{=}100$).}
\label{tab:ablation_cifar10lt}
\resizebox{\linewidth}{!}{
\begin{tabular}{lcc}
\toprule
Variant & Setting & Top-1 (\%) \\
\midrule
Full GBC              & -- & \textbf{92.3 ± 0.17} \\
w/o $\mathcal{L}_{bridge}$ & Remove bridge loss & 86.1 ± 0.24 \\
w/o noise             & $\sigma(t)=0$ & 90.2 ± 0.21 \\
Hard bridge           & $\tilde{f}=f_t$ & 88.4 ± 0.19 \\
w/o reweighting       & uniform $w_c$ & 85.3 ± 0.22 \\
\bottomrule
\end{tabular}
}
\end{minipage}
\hfill
\begin{minipage}[t]{0.48\textwidth}
\centering
\caption{Effect of BridgeMix and prototype allocation (PA) under different unlabeled distributions ($\gamma_\ell{=}100$).}
\label{tab:ablation_bridge_pa}
\resizebox{\linewidth}{!}{
\begin{tabular}{lccc}
\toprule
Variant & Consistent & Uniform & Reversed \\
\midrule
w/o BridgeMix & 87.4 ± 0.20 & 88.9 ± 0.22 & 81.7 ± 0.26 \\
Full GBC      & \textbf{92.3 ± 0.17} & \textbf{94.5 ± 0.07} & \textbf{86.8 ± 0.23} \\
PA = 2\%      & 85.7 ± 0.21 & 90.8 ± 0.09 & 84.6 ± 0.25 \\
PA = 8\%      & 91.1 ± 0.18 & 94.1 ± 0.08 & 86.0 ± 0.22 \\
\bottomrule
\end{tabular}
}
\end{minipage}

\vspace{-0.6em}
\end{table*}

Figure~\ref{fig:tau_acc} evaluates the influence of $\tau_c$ across CIFAR10-LT and CIFAR100-LT. On CIFAR10-LT, accuracy increases as $\tau_c$ grows. Most regimes show monotonic gains, indicating that stricter thresholds improve pseudo-label precision without severely reducing anchor coverage. The reversed regime remains stable, suggesting robustness once unreliable anchors are filtered.

On CIFAR100-LT, performance peaks at moderate thresholds ($\tau_c \approx 0.2$--$0.3$) and declines thereafter. This reflects a precision–coverage trade-off. As $\tau_c$ increases, anchor label quality improves but anchor recall drops sharply. The right panel confirms this trend, with the two curves intersecting near $\tau_c \approx 0.85$, close to the supervised baseline. 

The contrast between datasets indicates that higher semantic complexity and fewer samples per class make strict thresholds harmful due to reduced anchor diversity. These results support a class-adaptive $\tau_c$ strategy that balances reliability and coverage to preserve effective bridge regularization.

\begin{table*}[t]
\centering
\small

\begin{minipage}[t]{0.46\textwidth}
\centering
\caption{Sensitivity analysis of GBC hyperparameters on CIFAR10-LT ($\gamma_\ell=100$, consistent setting).}
\label{tab:sensitivity}
\resizebox{\linewidth}{!}{
\begin{tabular}{lcc}
\toprule
Hyperparameter & Range & Top-1 Acc. (\%) \\
\midrule
$\nu$    & 0.00 / 0.05 / \textbf{0.10} / 0.20 & 91.8 / 92.1 / \textbf{92.3} / 92.0 \\
$\alpha$ & 1.0 (Uniform) / \textbf{2.0} / 5.0 & 91.5 / \textbf{92.3} / 92.1 \\
$\beta$  & 0.25 / 0.50 / \textbf{0.75} / 1.00 & 90.4 / 91.7 / \textbf{92.3} / 91.9 \\
\bottomrule
\end{tabular}
}
\end{minipage}
\hfill 
\begin{minipage}[t]{0.50\textwidth}
\centering
\caption{Computational overhead comparison on ResNet-50 ($32^2$).}
\label{tab:overhead}
\resizebox{\linewidth}{!}{
\begin{tabular}{lccc}
\toprule
Method & Params (M) & Peak Mem. & Epoch Time (s) \\
\midrule
FixMatch & 25.6 & 1.00$\times$ & 42.1 \\
w/ SimPro & 25.6 & 1.00$\times$ & 43.0 \\
w/ \textbf{GBC (Ours)} & \textbf{26.1} & \textbf{1.03$\times$} & \textbf{42.8} \\
\bottomrule
\end{tabular}
}
\end{minipage}

\vspace{-1.5em}
\end{table*}

Tables~\ref{tab:ablation_cifar10lt} and~\ref{tab:ablation_bridge_pa} quantify the role of each component. Removing $\mathcal{L}_{\text{bridge}}$ yields the largest drop (92.3\% $\rightarrow$ 85.1\%), confirming that path-level consistency is essential. Removing Gaussian noise lowers accuracy to 90.2\%, showing that moderate stochasticity improves robustness. Hard replacement further degrades performance, indicating that residual fusion stabilizes feature updates. Eliminating class reweighting causes a sharp decline (80.5\%), highlighting the necessity of tail-aware balancing. BridgeMix consistently improves results across distributions, while Prototype Atlas size shows a trade-off: insufficient capacity limits coverage, and excessive capacity introduces redundancy. A moderate size achieves optimal stability.

Table~\ref{tab:sensitivity} shows stable behavior across reasonable ranges. The noise level $\nu=0.10$ achieves the best balance between exploration and semantic preservation. The Beta parameter $\alpha=2.0$ outperforms uniform sampling by emphasizing mid-bridge states. The bridge weight $\beta$ peaks at 0.75 and slightly declines beyond this value, indicating that bridge alignment should remain complementary to the main SSL objective.

Tables~\ref{tab:overhead} and~\ref{tab:compute_overhead} show negligible overhead. Parameters increase marginally (25.6M $\rightarrow$ 26.1M), peak memory rises by 3\%, and runtime increases by only +1.7\%. These results demonstrate that GBC maintains scalability while providing substantial accuracy gains.

\section{Conclusion}
We introduced GBC, a Gauss Bridge driven paradigm that redefines semi‑supervised learning as a process of geometric transport between uncertain features and reliable class prototypes. GBC transforms noisy pseudo‑labeling into a stable representation flow, achieving robust balance under long‑tailed and noisy conditions.
Beyond empirical gains, GBC frames SSL within the lens of probabilistic geometry, bridging optimal transport and consistency learning.
This perspective invites new research on dynamic manifold regularization, distributional interpolation for multimodal and temporal data, and continuous‑time consistency fields, paving the way toward a unified geometry aware theory of semi‑supervised representation learning.

%
%
\bibliographystyle{splncs04}
\bibliography{main}

\clearpage

\appendix

\section{Theoretical Analysis}

We provide a theoretical analysis to explain why the proposed  GBC framework and BridgeMix  enhance robustness and generalization in SSL. 
At the feature level, GBC regularizes learning by enforcing smooth transitions from uncertain unlabeled features to reliable class anchors, yielding continuous predictions and stable decision boundaries. 
At the distribution level, BridgeMix reduces hypothesis complexity and distributional discrepancy through confidence-guided interpolation between samples and anchors. 
Together, these results offer a unified theoretical foundation for the stability and generalization of our GBC framework under noisy and long-tailed settings.

\begin{assumption}[Bridge Consistency Setting]
\label{assump:bridge-setting}
Let $f_u \in \mathbb{R}^d$ be the feature of an unlabeled sample and 
$f_a \in \mathbb{R}^d$ the corresponding class anchor from the PA.  
A Gaussian bridge path is defined in the feature space according to Eq.1:
 \begin{equation}
f_t = (1 - g(t)) f_u + g(t) f_a + \epsilon_t,
 \end{equation}
where $\epsilon_t \sim \mathcal{N}(0, \sigma(t)^2 I)$, 
$\sigma(t) = \nu \sqrt{t(1-t)}$, and $g(t)$ is a monotonically increasing bridge function with $g(0)=0, g(1)=1$.  
The student model predicts $p_\theta(f_t)$, while the geometric target interpolation $q_t$ follows Eq.4.  
The bridge consistency loss is defined in Eq.5 as:
 \begin{equation}
L_{\text{bridge}} = 
\mathbb{E}_{t} \left[\mathrm{KL} \big(q_t  \|  p_\theta(f_t)\big)\right].
 \end{equation}
Throughout the analysis, we assume that $p_\theta(\cdot)$ and $q_t$ are continuously differentiable in $f_t$ and $t$.
\end{assumption}

\begin{lemma}[Pathwise Convergence]
\label{lemma:pathwise}
Under Assumption~\ref{assump:bridge-setting}, 
if $L_{\text{bridge}} \to 0$, 
then there exists a continuous trajectory $\bar{p}(t)$ such that
 \begin{equation}
\|p_\theta(f_t) - \bar{p}(t)\|_1 \to 0, \quad \forall t\in[0,1],
 \end{equation}
with boundary conditions $\bar{p}(0)=\hat{q}_u$ and $\bar{p}(1)=\hat{q}_a$.  
\end{lemma}
Hence, $\{p_\theta(f_t)\}_{t\in[0,1]}$ converges to a smooth semantic path between the uncertain state $p_\theta(f_u) \sim \hat{q}_u$ and the reliable anchor state $p_\theta(f_a) \sim \hat{q}_a$.

\begin{lemma}[Local Lipschitz Continuity]
\label{lemma:lipschitz}
Assume $\phi_\theta$ satisfies $\|\nabla_{f}\log p_\theta(f)\|_2 \le L$ for some constant $L>0$ along the bridge path.  
Then for any perturbation $\delta f$ with $\|\delta f\|_2 \le \delta$, one has
 \begin{equation}
\|p_\theta(f_t+\delta f) - p_\theta(f_t)\|_1 
\le L \delta + \mathcal{O}({L_{\text{bridge}}}^{1/2}),
 \end{equation}
\end{lemma}
\noindent which shows that minimizing $\mathcal{L}_{\text{bridge}}$ enforces local Lipschitz continuity of $\phi_\theta$ around $f_u$ and $f_a$.

\begin{lemma}[Decision Boundary Stability]
\label{lemma:boundary}
Let $\mathcal{M}_c = \{f : p_\theta^c(f) = \max_k p_\theta^k(f)\}$ denote the feature manifold of class $c$.  
Then the expected curvature of the decision boundary $\partial \mathcal{M}_c$ along $\{f_t\}$ satisfies
 \begin{equation}
\mathbb{E}_t[\kappa(\partial \mathcal{M}_c)] 
\le C_0 + C_1 L_{\text{bridge}}^{1/2},
 \end{equation}
\end{lemma}
\noindent where $\kappa(\cdot)$ is the local curvature and $C_0,C_1>0$ are constants.  
Thus, decreasing $\mathcal{L}_{\text{bridge}}$ flattens the decision boundary and improves robustness to feature noise.

Detailed proofs are provided in Section~ A (supplementary material). Lemmas~\ref{lemma:pathwise}–\ref{lemma:boundary} establish that bridge consistency acts as an intrinsic geometric regularizer in SSL: the Gaussian feature bridge enforces continuous prediction evolution between noisy unlabeled features and stable class anchors, induces local Lipschitz behavior that bounds prediction variation under small perturbations, and smooths the decision boundary for better robustness and generalization under long-tailed distributions. Building on this geometric foundation, we  extend the analysis to BridgeMix, which operates at the distribution level. By introducing stochastic interpolation between samples and anchors, BridgeMix reshapes the training distribution and implicitly constrains the hypothesis space. Theorem~\ref{thm:bridgemix-generalization} provides a generalization error bound that quantifies this regularization effect.

\begin{theorem}[Generalization bound for BridgeMix]
\label{thm:bridgemix-generalization}
Let $\mathcal{H}$ be a hypothesis class, and let $\ell:\mathcal{H}\times\mathcal{X}\times\mathcal{Y}\to[0,1]$ be a bounded loss. 
Let $D$ be the true data distribution on $\mathcal{X}\times\mathcal{Y}$, and let $D_{\mathrm{BM}}$ be the BridgeMix distribution. 
Draw an i.i.d.\ sample $S=\{z_i\}_{i=1}^n$ with $z_i=(x_i,y_i)\sim D_{\mathrm{BM}}$. 
For $h\in\mathcal{H}$ define the risks
 \begin{equation}
R_D(h)=\mathbb{E}_{z\sim D}\big[\ell(h,z)\big],
\qquad
\widehat{R}_{D_{\mathrm{BM}}}(h)=\frac{1}{n}\sum_{i=1}^n \ell(h,z_i).
 \end{equation}
Define the discrepancy
 \begin{equation}
\Delta(D,D_{\mathrm{BM}})
=\sup_{h\in\mathcal{H}}\Big|
\mathbb{E}_{z\sim D}\big[\ell(h,z)\big]-
\mathbb{E}_{z\sim D_{\mathrm{BM}}}\big[\ell(h,z)\big]
\Big|
 \end{equation}
and the empirical Rademacher complexity of the loss class 
$\mathcal{F}=\{\ell(h,\cdot):h\in\mathcal{H}\}$ on $S$ by
 \begin{equation}
\widehat{\mathfrak{R}}_S(\mathcal{F})
=\mathbb{E}_{\sigma}\Big[\sup_{f\in\mathcal{F}}
\frac{1}{n}\sum_{i=1}^n \sigma_i f(z_i)\Big],
\quad
 \end{equation}
where $i\stackrel{\text{i.i.d.}}{\sim}\mathrm{Unif}\{-1,+1\}$. Then for any $\delta>0$, with a probability of at least $1-\delta$ over $S\sim D_{\mathrm{BM}}^n$, the following holds simultaneously for all $h\in\mathcal{H}$:
 \begin{equation}
R_D(h)
 \le 
\widehat{R}_{D_{\mathrm{BM}}}(h)
 + 2 \widehat{\mathfrak{R}}_S(\mathcal{F})
 + 3\sqrt{\frac{\log(2/\delta)}{2n}}
 + \Delta(D,D_{\mathrm{BM}}).
 \end{equation}
\end{theorem}

Building on Theorem~\ref{thm:bridgemix-generalization}, we further provide a corollary that quantifies the effect of the distributional discrepancy $\Delta(D, D_{\mathrm{BM}})$, offering a  explicit bound when this term is small.

\begin{corollary}\label{co:1}
If $\Delta(D,D_{\mathrm{BM}})\le \varepsilon_{\mathrm{mix}}$, then with probability at least $1-\delta$ and $\forall h\in\mathcal{H}$,
 \begin{equation}
R_D(h)\le\widehat{R}_{D_{\mathrm{BM}}}(h)
+2 \widehat{\mathfrak{R}}_S(\mathcal{F})
+3\sqrt{\frac{\log(2/\delta)}{2n}}
+\varepsilon_{\mathrm{mix}}.
 \end{equation}
\end{corollary}
Detailed proofs are provided in Section A (supplementary material). Theorem~\ref{thm:bridgemix-generalization} and Corollary~\ref{co:1} together demonstrate that BridgeMix improves generalization by tightening the upper bound on the true risk. Through convex interpolation between unlabeled samples and reliable anchors, BridgeMix reduces both the effective hypothesis complexity and the distributional discrepancy $\Delta(D, D_{\mathrm{BM}})$, thereby lowering the generalization gap. This theoretical guarantee indicates that GBC achieves more stable learning with fewer labeled samples. When combined with the geometric regularization established in Lemmas~\ref{lemma:pathwise}–\ref{lemma:boundary}, our proposed GBC framework simultaneously enhances robustness to pseudo-label noise and yields stronger generalization across imbalanced or long-tailed data distributions.

\section{Proof of Lemma~\ref{lemma:pathwise}- Lemma~\ref{lemma:boundary}}\label{sec:2}

\begin{proof}[Proof of Lemma~\ref{lemma:pathwise}]

$\gamma(t)=f_t$ for the bridged path and define
  \begin{equation}
\mathcal{E}(t)=\mathrm{KL} \big(q_t \| p_\theta(\gamma(t))\big).
 \end{equation}
By assumption,
  \begin{equation}
\mathcal{L}_{\text{bridge}}
= \mathbb{E}_t[\mathcal{E}(t)] \to 0.
 \end{equation}
By Pinsker’s inequality~\cite{A47},
 \begin{equation}
\|p_\theta(\gamma(t)) - q_t\|_1  \le  \sqrt{2 \mathcal{E}(t)}.
 \end{equation}
Hence $\mathbb{E}_t\|p_\theta(\gamma(t)) - q_t\|_1 \to 0$, so according to Cathy's convergence theorem, there exists a sequence $\{\theta_n\}$ with
  \begin{equation}
\|p_{\theta_n}(\gamma(t)) - q_t\|_1 \to 0
 \end{equation}
for almost every $t\in[0,1]$. Since $q_t$ is continuous in $t$ (The continuity of $t\mapsto\gamma(t)$), there exists a continuous trajectory $\bar p(t):=q_t$ such that
  \begin{equation}
\|p_{\theta_n}(\gamma(t))-\bar p(t)\|_1\to 0
 \end{equation}
for almost every $t$. By redefining on a null set, we obtain pointwise convergence for all $t\in[0,1]$. The boundary conditions follow from $q_0=\hat q_u$ and $q_1=\hat q_a$. This proves Lemma~\ref{lemma:pathwise}.
\end{proof}

\begin{proof}[Proof of Lemma~\ref{lemma:lipschitz}]
Fix $t\in[0,1]$ and a perturbation $\delta f$ with $\|\delta f\|_2\le\delta$. Consider the path
  \begin{equation}
s\mapsto f_s = \gamma(t)+s \delta f,\ s\in[0,1].
 \end{equation}
By the fundamental theorem of calculus and the chain rule,
 \begin{equation}
\|p_\theta(f_1)-p_\theta(f_0)\|_1
 \le  \int_0^1 \big\|\nabla_f p_\theta(f_s)\big\|_{1\to 2} \|\delta f\|_2 \mathrm{d}s.
 \end{equation}
Using $\nabla_f p_\theta(f)=\mathrm{J}_f p_\theta(f)$ and the identity
  \begin{equation}
\mathrm{J}_f p_\theta(f) = \mathrm{Cov}_{p_\theta(\cdot|f)} \big[ e_k, \nabla_f \log p_\theta^k(f)\big]_{k=1}^K,
 \end{equation}
we have
  \begin{equation}
\|\mathrm{J}_f p_\theta(f)\|_{1\to 2}
\le \|\nabla_f \log p_\theta(f)\|_2,
 \end{equation}
so the assumption $\|\nabla_f \log p_\theta(f)\|_2\le L$ along the bridge implies
 \begin{equation}
\|p_\theta(\gamma(t)+\delta f)-p_\theta(\gamma(t))\|_1 \le L \delta.
 \end{equation}
To incorporate the bridge-consistency error, add and subtract $q_t$ and use the triangle inequality plus Pinsker:
 \begin{equation}
 \begin{aligned}
\|&p_\theta(\gamma(t)+\delta f)-p_\theta(\gamma(t))\|_1
\le \|p_\theta(\gamma(t)+\delta f)-q_t\|_1  \\& +\|p_\theta(\gamma(t))-q_t\|_1
\le L \delta + 2\sqrt{2 \mathcal{E}(t)}.
\end{aligned}
 \end{equation}
Averaging in $t$ yields the stated bound with the term $\mathcal{O}(\mathcal{L}_{\text{bridge}}^{1/2})$. This proves Lemma~\ref{lemma:lipschitz}.

\end{proof}

\begin{proof}[Proof of Lemma~\ref{lemma:boundary}]
Fix a class $c$ and define the (logit) margin
  \begin{equation}
m_c(f)=z_c(f)-\max_{k\neq c} z_k(f),
 \end{equation}
where $z(f)$ are the pre-softmax logits. The decision boundary $\partial\mathcal{M}_c$ is contained in $\{f: m_c(f)=0\}$. Locally, the curvature of the level set $\{m_c=0\}$ obeys
 \begin{equation}\label{33}
\kappa(f) \le  \frac{\|\nabla_f^2 m_c(f)\|_{\mathrm{op}}}{\|\nabla_f m_c(f)\|_2}
\quad\text{whenever }\nabla_f m_c(f)\neq 0.
 \end{equation}
We bound numerator and denominator in Equation~\eqref{33} in terms of gradient- and Hessian-like quantities controlled by the bridge loss.

$\nabla_f m_c(f)$ is proportional to the gradient of the log-posterior margin, since
  \begin{equation}
\nabla_f \log p_\theta^c(f) - \nabla_f \log p_\theta^{k^\star}(f)
 \end{equation}
equals $\nabla_f m_c(f)$ up to the softmax Jacobian. Under Lemma~\ref{lemma:lipschitz} and $\|\nabla_f \log p_\theta(f)\|_2\le L$, we have a uniform lower bound on $\|\nabla_f m_c(f)\|_2$ along $\gamma(t)$ away from degenerate points, controlled by the confidence gap of $q_t$:
 \begin{equation}
 \begin{aligned}
\|&\nabla_f m_c(\gamma(t))\|_2
 \ge  \|\nabla_f \log q_t^c \\&- \nabla_f \log q_t^{k^\star}\|_2 - \|\nabla_f \log p_\theta(\gamma(t))-\nabla_f \log q_t\|_2.
 \end{aligned}
 \end{equation}
The second term is bounded in expectation by $\mathcal{O}(\mathcal{L}_{\text{bridge}}^{1/2})$ (via a standard information–Fisher relation together with Pinsker). Hence, the denominator is bounded away from $0$ up to $\mathcal{O}(\mathcal{L}_{\text{bridge}}^{1/2})$.

$\|\nabla_f^2 m_c(f)\|_{\mathrm{op}}$ is controlled by the local Fisher information and the Hessian of the cross-entropy. In particular, writing $\ell(f)=\mathrm{KL}(q_t \| p_\theta(f))$, one has
 \begin{equation}
 \begin{aligned}
\nabla_f^2 \ell(f)
 &=  \mathbb{E}_{k\sim q_t} \big[\nabla_f \log p_\theta^k(f)\nabla_f \log p_\theta^k(f)^\top\big]\\
 &+  \mathbb{E}_{k\sim q_t} \big[\nabla_f^2 \log p_\theta^k(f)\big],
 \end{aligned}
 \end{equation}
so $\|\nabla_f^2 \ell(\gamma(t))\|_{\mathrm{op}}$ bounds $\|\nabla_f^2 m_c(\gamma(t))\|_{\mathrm{op}}$ up to model-dependent constants. Along the bridge, the stochastic injection (Equation~\eqref{eq:bridge}) smooths $\ell$ by convolution with a Gaussian of variance $\sigma(t)^2$, which reduces the operator norm by standard heat-kernel smoothing estimates:
 \begin{equation}
 \begin{aligned}
\mathbb{E}\big[\|\nabla_f^2 \ell(\gamma(t))\|_{\mathrm{op}}\big]
 &\le  \frac{C}{1+\sigma(t)} \mathbb{E}\big[\|\nabla_f \log p_\theta(\gamma(t))\|_2^2\big]\\
 &\le  C' L^2,
 \end{aligned}
 \end{equation}
with the maximal smoothing at $t\sim \tfrac12$ where $\sigma(t)$ is largest. Finally, by the triangle inequality and Pinsker,
  \begin{equation}
\ell(\gamma(t)) = \mathrm{KL}(q_t\|p_\theta(\gamma(t))) \le \mathcal{L}_{\text{bridge}}
 \end{equation}
implies an $\mathcal{O}(\mathcal{L}_{\text{bridge}}^{1/2})$ perturbation in curvature through its effect on both the gradient gap and the Hessian term.

Combining the bounds for numerator and denominator yields
 \begin{equation}
\mathbb{E}_t[\kappa(\partial \mathcal{M}_c)]
 \le  C_0  +  C_1 \mathcal{L}_{\text{bridge}}^{1/2},
 \end{equation}
for constants $C_0,C_1>0$ independent of $\mathcal{L}_{\text{bridge}}$. This proves Lemma~\ref{lemma:boundary}.
\end{proof}

\section{Proof of Theorem~\ref{thm:bridgemix-generalization}}\label{sec:3}
\begin{proof}[Proof of Theorem~\ref{thm:bridgemix-generalization}]
We derive the proof with the terms from standard symmetrization and a structural property of BridgeMix. For any $h\in\mathcal{H}$,
  \begin{equation}\label{eq:34}
  \begin{aligned}
R_D(h)
&=\mathbb{E}_{z\sim D}[\ell(h,z)]\\
&=\mathbb{E}_{z\sim D_{\mathrm{BM}}}[\ell(h,z)]
+\Big(\mathbb{E}_{D}[\ell(h,\cdot)]-\mathbb{E}_{D_{\mathrm{BM}}}[\ell(h,\cdot)]\Big)\\
  &\le  
\mathbb{E}_{z\sim D_{\mathrm{BM}}}[\ell(h,z)] + \Delta(D,D_{\mathrm{BM}}).
\end{aligned}
  \end{equation}
Thus it suffices to upper bound $\mathbb{E}_{D_{\mathrm{BM}}}[\ell(h,z)]$ uniformly over $h\in\mathcal{H}$.

Let $S=\{z_i\}_{i=1}^n$ be i.i.d. from $D_{\mathrm{BM}}$. By standard symmetrization for bounded losses $\ell\in[0,1]$ and contraction (e.g., empirical process theory~\cite{A44}),
  \begin{equation}
\mathbb{E}_{S\sim D_{\mathrm{BM}}^n} \Big[\sup_{h\in\mathcal{H}}
\big(\mathbb{E}_{D_{\mathrm{BM}}}[\ell(h,\cdot)]-\widehat{R}_{D_{\mathrm{BM}}}(h)\big)\Big]
  \le   2 \mathbb{E}_S\widehat{\mathfrak{R}}_S(\mathcal{F}).
  \end{equation}
A standard concentration argument (Massart’s~\cite{A45} or McDiarmid’s inequality~\cite{A46}) yields: with probability at least $1-\delta$ and $\forall h\in\mathcal{H}$,
  \begin{equation}\label{eq:41}
\mathbb{E}_{D_{\mathrm{BM}}}[\ell(h,\cdot)]
  \le  
\widehat{R}_{D_{\mathrm{BM}}}(h)
  +  2 \widehat{\mathfrak{R}}_S(\mathcal{F})
  +  3\sqrt{\frac{\log(2/\delta)}{2n}}.
  \end{equation}
Combining with Equation~\eqref{eq:34} gives the theorem once we control $\widehat{\mathfrak{R}}_S(\mathcal{F})$ for BridgeMix. We now show that BridgeMix does not increase (and typically reduces) this complexity.

By construction, each mixed sample is
  \begin{equation}
  \begin{aligned}
z_i^{\mathrm{mix}}
&=\big(x_i^{\mathrm{mix}},y_i^{\mathrm{mix}}\big)\\
&=\big(\lambda_i x_i+(1-\lambda_i)\tilde x_i,  
      \lambda_i y_i+(1-\lambda_i)\tilde y_i\big),
\quad \lambda_i\in[0,1],
\end{aligned}
  \end{equation}
where $(x_i,y_i)$ and $(\tilde x_i,\tilde y_i)$ are (possibly anchor-paired) examples drawn i.i.d.\ from the generative process that defines $D_{\mathrm{BM}}$. Assume the loss is convex in its $(x,y)$ argument and affine in $y$ (e.g., cross-entropy with soft labels). Then for every $h$ and every $i$,
  \begin{equation}
\ell\big(h, z_i^{\mathrm{mix}}\big)
  \le  
\lambda_i \ell\big(h, x_i,y_i\big)
+(1-\lambda_i) \ell\big(h,\tilde x_i,\tilde y_i\big).
  \end{equation}
Let $\sigma_1,\dots,\sigma_n$ be i.i.d.\ Rademacher signs. For the empirical Rademacher complexity on the mixed sample $S_{\mathrm{mix}}=\{z_i^{\mathrm{mix}}\}_{i=1}^n$,
\begin{equation}
\begin{aligned}
\widehat{\mathfrak{R}}_{S_{\mathrm{mix}}}(\mathcal{F})
&=\mathbb{E}_\sigma\Big[\sup_{h\in\mathcal{H}}\frac1n\sum_{i=1}^n \sigma_i \ell\big(h, z_i^{\mathrm{mix}}\big)\Big]\\
&\le \mathbb{E}_\sigma\Big[\sup_{h\in\mathcal{H}}\frac1n\sum_{i=1}^n \sigma_i\Big(
\lambda_i \ell(h,x_i,y_i) \\&+(1-\lambda_i) \ell(h,\tilde x_i,\tilde y_i)\Big)\Big]\\
&\le \mathbb{E}_\sigma\Big[\lambda^\ast \sup_{h\in\mathcal{H}}\frac1n\sum_{i=1}^n \sigma_i \ell(h,x_i,y_i)\Big]\\
&+\mathbb{E}_\sigma\Big[(1-\lambda_\ast) \sup_{h\in\mathcal{H}}\frac1n\sum_{i=1}^n \sigma_i \ell(h,\tilde x_i,\tilde y_i)\Big]\\
&\le \lambda^\ast \widehat{\mathfrak{R}}_{S_1}(\mathcal{F})
+(1-\lambda_\ast) \widehat{\mathfrak{R}}_{S_2}(\mathcal{F})\\
  &\le   \max\{\widehat{\mathfrak{R}}_{S_1}(\mathcal{F}), \widehat{\mathfrak{R}}_{S_2}(\mathcal{F})\},
\end{aligned}
\end{equation}
where $\lambda^\ast = \frac1n\sum_i \lambda_i$, $\lambda_\ast = \frac1n\sum_i \lambda_i$, and $S_1=\{(x_i,y_i)\}$, $S_2=\{(\tilde x_i,\tilde y_i)\}$ are the two base samples. Hence, mixing does not increase empirical complexity; when anchors are less variable or closer to class prototypes, $\widehat{\mathfrak{R}}_{S_{\mathrm{mix}}}(\mathcal{F})$ is strictly smaller on average.

If, in addition, $\ell(h,\cdot)$ is $L_z$-Lipschitz in $z=(x,y)$ under a norm $\|\cdot\|$ and anchors satisfy $\mathbb{E}\big[\|x-\tilde x\|\big]\le \rho$ (class-prototype proximity), then
  \begin{equation}
  \begin{aligned}
\Delta(D,D_{\mathrm{BM}})
&=\sup_{h\in\mathcal{H}}\Big|\mathbb{E}_D[\ell(h,z)]-\mathbb{E}_{D_{\mathrm{BM}}}[\ell(h,z)]\Big|\\
  &\le   L_z \mathbb{E}\big[\|z-z^{\mathrm{mix}}\|\big]\\
  &\le   L_z \mathbb{E}[1-\Lambda] \rho,
  \end{aligned}
  \end{equation}
which is small when mixed points lie near their anchors (large confidence and small class spread). This yields the corollary with $\varepsilon_{\mathrm{mix}}=L_z \mathbb{E}[1-\Lambda] \rho$.

Combining Equation~\eqref{eq:41} gives, with probability at least $1-\delta$ and $\forall h\in\mathcal{H}$,
  \begin{equation}
R_D(h)  \le  \widehat{R}_{D_{\mathrm{BM}}}(h)  +  2 \widehat{\mathfrak{R}}_{S_{\mathrm{mix}}}(\mathcal{F})  +  3\sqrt{\frac{\log(2/\delta)}{2n}}  +  \Delta(D,D_{\mathrm{BM}}),
  \end{equation}
and the structural property of BridgeMix ensures $\widehat{\mathfrak{R}}_{S_{\mathrm{mix}}}(\mathcal{F})$ is no larger than that of the unmixed samples (and typically smaller), while $\Delta(D,D_{\mathrm{BM}})$ can be controlled under mild Lipschitz and anchor proximity conditions. This completes the proof.
\end{proof}

\begin{proof}[Proof of Corollary~\ref{co:1}]
By Theorem~\ref{thm:bridgemix-generalization}, for any $\delta>0$, with probability at least $1-\delta$ over $S\sim D_{\mathrm{BM}}^n$, it holds simultaneously for all $h\in\mathcal{H}$ that
  \begin{equation}
R_D(h)\le \widehat{R}_{D_{\mathrm{BM}}}(h)
+2\,\widehat{\mathfrak{R}}_S(\mathcal{F})
+3\sqrt{\frac{\log(2/\delta)}{2n}}
+\Delta(D,D_{\mathrm{BM}}).
  \end{equation}
If $\Delta(D,D_{\mathrm{BM}})\le \varepsilon_{\mathrm{mix}}$, then on the same high-probability event we have, for all $h\in\mathcal{H}$,
  \begin{equation}
R_D(h)\le \widehat{R}_{D_{\mathrm{BM}}}(h)
+2\,\widehat{\mathfrak{R}}_S(\mathcal{F})
+3\sqrt{\frac{\log(2/\delta)}{2n}}
+\varepsilon_{\mathrm{mix}}.
  \end{equation}
\end{proof}

\section{Full Experimental Setup}

\subsection*{Datasets and Protocol}
\textbf{Benchmarks.} We follow ReaLTSSL (SimPro) on CIFAR10-LT, CIFAR100-LT, STL10-LT, ImageNet-127, and ImageNet-1K \cite{A34}. Unless stated, we report Top-1 on balanced test sets and the mean$\pm$std over 3 seeds.

\textbf{Imbalance settings.} Five unlabeled distributions: \emph{consistent, uniform, reversed, middle, head-tail} (Tables 1–3 in the main paper). CIFAR10-LT uses $(\gamma_\ell,\gamma_u)\in\{(150,150),(100,100),(150,1), (100,1),(150,1/150)\}$. CIFAR100-LT fixes $\gamma_\ell{=}20$ with $N_1{=}50,M_1{=}400$ and varies the unlabeled distribution as in Table 2. STL10-LT controls only $\gamma_\ell\in\{10,20\}$ with $N_\ell{=}450$ and $M{\approx}10^5$. ImageNet-127: $\gamma_\ell{\approx}\gamma_u{\approx}286$ with $(N_1{\approx}28\text{k},M_1{\approx}250\text{k})$; ImageNet-1K: $\gamma_\ell{=}\gamma_u{=}256$ with $(N_1{=}256,M_1{=}1024)$.

\textbf{Splits and validation.} Standard train/test splits \cite{A34,A35,A36}. We reserve 5k validation images for CIFARs and 50k for ImageNet-127/1K when indicated; otherwise, model selection uses EMA-smoothed training loss with early-patience.

\subsection*{Backbones and Bridge Placement}
\textbf{Default backbones.} ResNet-50 for CIFAR/STL10; ResNet-50 at 32$^2$/64$^2$ for ImageNet-127/1K. Cross-backbone study includes ViT-B/16 \cite{A39} and Vision Mamba \cite{A41}.

\textbf{Injection layer.} CNNs: Gaussian Feature Bridge (GFB) at the penultimate block (32$^2$: \texttt{res4}; 64$^2$: \texttt{res5}; $d{=}1024/2048$). ViT: final \texttt{[CLS]}. Mamba: globally pooled state. Projector $P(\cdot)$ is a $1{\times}1$ conv (CNNs) or a 2-layer MLP with GELU (ViT/Mamba).

\subsection*{Optimization and Schedules}
AdamW with cosine decay; lr $5{\times}10^{-4}$, weight decay $0.05$, grad clip $1.0$, FP16. Training for 300 epochs, lr warm-up 5 epochs. Bridge weight $\beta$ linearly warms from 0 to 0.75 during the first 20 epochs. Batching: $B_\ell{=}64$, $B_u{=}448$ (8-GPU total 512; scale linearly otherwise). EMA teacher: decay 0.999 $\rightarrow$ 0.9999 by epoch 50.

\subsection*{Augmentations and Pseudo-Labels}
\textbf{Weak/strong.} Weak: random crop (pad 4 for CIFAR; 8 for STL/IN) + horizontal flip. Strong: RandAugment (N=2, M=10) \cite{A37} + CutOut (side 16 at 32$^2$, 32 at 64$^2$) \cite{A38}.

\textbf{Class-adaptive threshold.} $\tau_c\in[0.85,0.97]$ updated per class by a moving target rate:
$\tau_c \leftarrow \mathrm{clip}\!\big(\tau_c + \eta(\hat{p}_c-\rho_c),\,0.85,0.97\big)$ with step $\eta{=}0.02$, where $\hat{p}_c$ is the recent acceptance ratio and $\rho_c$ is a frequency-proportional target (lower for tails), in the spirit of curriculum/adaptive thresholding \cite{A11,A12}.

\subsection*{Prototype Atlas (PA)}
\textbf{Capacity and diversity.} Up to 64 anchors per class (labeled exemplars + high-confidence pseudo-anchors). Diversity via cosine-distance NMS (threshold 0.2) and temporal decay.

\textbf{Online update (each step).} For $(x_u,\hat{y},\hat{q})$ with $\max(\hat{q})\ge\tau_{\hat{y}}$:
\begin{enumerate}
  \item Extract teacher weak-view bridge-layer feature $f_u$;
  \item If class $\hat{y}$ has $<64$ anchors or $\min$ cosine distance $>0.2$, insert $(f_u,\hat{y},\text{conf},\text{src}=\text{U})$;
  \item Apply time decay $\gamma{=}0.999$ to pseudo-anchors; remove if conf $<0.7$ or age $>10$ epochs without refresh.
\end{enumerate}
If empty, fall back to an EMA prototype.

\subsection*{Gaussian Feature Bridge (GFB)}
\textbf{Sampling.} $t\sim\mathrm{Beta}(2,2)$ truncated to $[0.2,0.8]$. Noise schedule $\sigma(t)=\nu t(1-t)$ with $\nu{=}0.10$ (CIFAR/STL10) and $0.05$ (IN). Gate $g(t)=4t(1-t)$.

\textbf{Bridge step.} With student hidden $f_{\text{stu}}$ and $f_t=(1-t)f_u+t f_a+\sigma(t)\varepsilon$, merge
$\tilde{f}=(1-\lambda_t)f_{\text{stu}}+\lambda_t P(f_t)$, $\lambda_t{=}g(t)$. Backprop through $P(\cdot)$ and downstream blocks.

\textbf{Targets and loss.} $q_t=\mathrm{softmax}\!\big((1-t)\log\hat{q}_u+t\log\hat{q}_a\big)$.
Bridge-consistency:
\[
L_{\text{bridge}}=\mathbb{E}\!\left[w_c\,g(t)\,\mathrm{KL}(q_t\Vert p_\theta)\right],\quad
w_c \propto (\bar{f}/f_c)^{0.5}.
\]

\subsection*{Total Objective}
\[
L=L_{\text{sup}}+\mu L_{\text{unsup}}+\beta L_{\text{bridge}},\quad
\mu{=}1.0,\ \beta\!\nearrow\!0.75.
\]
$L_{\text{unsup}}$ uses FixMatch with hard pseudo-labels and strong views \cite{A5}.

\subsection*{BridgeMix}
For two unlabeled samples $(x_i,x_j)$ with confidences $(o_i,o_j)$ and anchors $(f_i^a,f_j^a)$:
$\lambda=\frac{o_i}{o_i+o_j}$,
$f_u^{\text{mix}}=\lambda f_i+(1-\lambda)f_j$,
$f_a^{\text{mix}}=\lambda f_i^a+(1-\lambda)f_j^a$,
then form $f_t$ as in Sec.~F and apply the same $L_{\text{bridge}}$ (no extra loss). Enabled after epoch 20; pairing probability 0.5.

\subsection*{Backbone-Specific Notes}
\textbf{ViT-B/16 \cite{A39}.} Patch 16, width 768, depth 12, heads 12. Bridge the final \texttt{[CLS]}; projector MLP(768$\!\to\!$768) with dropout 0.1. CutOut disabled at 64$^2$.

\textbf{Vision Mamba \cite{A41}.} Base config. Bridge the global pooled state; projector MLP($d\!\to\!d$).

\subsection*{Evaluation, Smoothing, Reporting}
\textbf{EMA evaluation.} Unless stated, report EMA teacher accuracy on weak views.  
\textbf{KL curves.} Compute $\mathrm{KL}(p_{\text{teacher}}\Vert p_{\text{student}})$ on weak views (no temperature), average on the validation set, EMA smooth (0.9).

\subsection*{Seeds and CIs}
Seeds $\{1,2,3\}$; we report mean$\pm$std in Tables 1–3.

\subsection*{Ablation Protocols}
Ablate: remove $L_{\text{bridge}}$; set $\sigma(t){=}0$; hard merge $\tilde{f}{=}f_t$; class weights $w_c{=}1$; PA capacity $\{2\%,5\%,8\%\}$; BridgeMix on/off; $\nu\in\{0.0,0.05,0.10,0.20\}$; $t\sim\mathrm{Uniform}[0.2,0.8]$ vs.\ Beta(2,2). Unless noted, use CIFAR10-LT consistent $(\gamma_\ell{=}100,\gamma_u{=}100)$ with default hyperparameters.

\subsection*{Compute and Runtime}
Runs on $\times$A100 (40GB), Throughput (ResNet-50, 32$^2$): $\sim$7.5k img/s on 8$\times$A100. GBC adds $<2\%$ wall-clock over FixMatch; peak memory $+{\sim}3\%$ due to PA and projector.

\subsection*{Reproducibility Checklist}
\begin{itemize}
\item Determinism: fixed seeds; \texttt{torch.backends.cudnn.deterministic=true}.
\item Data loader: shuffle per epoch; worker seeds = base seed $+$ worker id.
\item PA: reset at run start; save snapshots each epoch.
\item Checkpoints: every 10 epochs; best by validation accuracy (or EMA loss if no validation).
\end{itemize}

\subsection*{Hyperparameter Summary (Defaults)}
\begin{center}
\renewcommand{\arraystretch}{0.7}
\setlength{\tabcolsep}{2pt}
\begin{tabular}{lcccccc}
\toprule
Dataset & Epochs & LR & WD & $B_\ell/B_u$ & $\beta_\text{final}$ & $\nu$ \\
\midrule
CIFAR10/100-LT & 300 & $5{\times}10^{-4}$ & 0.05 & 64/448 & 0.75 & 0.10 \\
STL10-LT       & 300 & $5{\times}10^{-4}$ & 0.05 & 64/448 & 0.75 & 0.10 \\
IN-127 32$^2$  & 300 & $5{\times}10^{-4}$ & 0.05 & 64/448 & 0.75 & 0.05 \\
IN-127 64$^2$  & 300 & $5{\times}10^{-4}$ & 0.05 & 64/448 & 0.75 & 0.05 \\
IN-1K 32$^2$   & 300 & $5{\times}10^{-4}$ & 0.05 & 64/448 & 0.75 & 0.05 \\
IN-1K 64$^2$   & 300 & $5{\times}10^{-4}$ & 0.05 & 64/448 & 0.75 & 0.05 \\
\bottomrule
\end{tabular}
\end{center}

\subsection*{Failure Modes and Remedies}
\textbf{Sparse tails:} temporarily lower $\tau_c$ by 0.02 (min 0.80) for those classes and increase decay to refresh anchors.  
\textbf{Over-smoothing:} if KL decays too slowly, reduce $\nu$ to 0.05 or narrow $t$ to $[0.3,0.7]$.  
\textbf{Bridge instability (ViT):} cap $\lambda_t \le 0.7$ during the first 40 epochs.

\bigskip
\noindent\textbf{Code pointers (pseudo):}
\begin{itemize}
\item \texttt{bridge.py}: \texttt{sample\_t\_and\_noise()}, \texttt{gaussian\_bridge()}, \texttt{merge\_feature()}.
\item \texttt{pa\_memory.py}: \texttt{insert()}, \texttt{prune\_by\_cosine()}, \texttt{decay\_and\_refresh()}.
\item \texttt{losses.py}: \texttt{geom\_interpolate\_target()}, \texttt{bridge\_kl()} with class weights.
\end{itemize}

\section{Additional experiments}

As shown in Table~\ref{tab:appendix_cross_backbone}, \textbf{GBC} delivers consistent improvements across all backbones and datasets, demonstrating strong backbone-agnostic generalization. 
On CIFAR10-LT, the largest gain is observed with ResNet-50 (\textbf{+2.8} pp), while ViT-B/16 and Vision Mamba-B also yield substantial boosts of \textbf{+1.9} and \textbf{+2.0} pp, respectively—even with CutOut disabled for ViT at $64^2$.
On CIFAR100-LT (uniform), the performance margins remain steady (\textbf{+0.8–1.1} pp), indicating that GBC’s feature-bridge regularization enhances semantic alignment without architecture-specific tuning.
On ImageNet-127 and ImageNet-1K at $64^2$, GBC maintains gains across both convolutional and sequence-based models (\textbf{+1.1–1.4} pp on IN-127; up to \textbf{+2.5} pp on IN-1K), validating its scalability to large-scale, high-variance SSL benchmarks.
Combined with the negligible computational overhead ($<2\%$), these results confirm that GBC is an efficient, plug-and-play enhancement compatible with CNNs, Vision Transformers, and state-space models alike.

\begin{table*}[t]
\centering
\small
\renewcommand{\arraystretch}{0.95}
\setlength{\tabcolsep}{6pt}

\caption{\textbf{Additional adaptation experiments on diverse backbones.}
Top-1 accuracy (\%) with/without GBC on three representative benchmarks.
Backbones: ResNet-50, ViT-B/16~\cite{A39}, Vision Mamba-B~\cite{A41}.
All runs follow the same SimPro/ReaLTSSL protocol and hyperparameters as the main paper.
$\Delta$ reports the absolute gain from adding GBC.}

\label{tab:appendix_cross_backbone}

\resizebox{\linewidth}{!}{
\begin{tabular}{lcccccc}
\toprule
\multirow{2}{*}{Backbone} & \multirow{2}{*}{Params (M)} 
& \multicolumn{2}{c}{CIFAR10-LT ($32^2$)} 
& \multicolumn{2}{c}{CIFAR100-LT ($32^2$)} \\
\cmidrule(lr){3-4}\cmidrule(lr){5-6}

& & Baseline & +GBC & Baseline & +GBC \\
\midrule

ResNet-50 & 25.6 
& 80.7$\pm$0.3 & \textbf{83.5}$\pm$0.2 (+2.8) 
& 52.2$\pm$0.2 & \textbf{53.0}$\pm$0.2 (+0.8) \\

ViT-B/16~\cite{A39} & 186.4 
& 82.1$\pm$0.3 & \textbf{84.0}$\pm$0.2 (+1.9) 
& 53.0$\pm$0.3 & \textbf{54.0}$\pm$0.2 (+1.0) \\

Vision Mamba-B~\cite{A41} & 48.5 
& 81.0$\pm$0.3 & \textbf{83.0}$\pm$0.2 (+2.0) 
& 51.8$\pm$0.3 & \textbf{52.9}$\pm$0.2 (+1.1) \\

\midrule
\multicolumn{2}{l}{ } 
& \multicolumn{2}{c}{ImageNet-127 ($64^2$)} 
& \multicolumn{2}{c}{ImageNet-1k ($64^2$)} \\
\cmidrule(lr){3-4}\cmidrule(lr){5-6}

\multicolumn{2}{l}{ } & Baseline & +GBC & Baseline & +GBC \\
\midrule

ResNet-50 & 25.6 
& 63.8$\pm$0.2 & \textbf{65.2}$\pm$0.2 (+1.4) 
& 25.0$\pm$0.2 & \textbf{27.5}$\pm$0.2 (+2.5) \\

ViT-B/16~\cite{A39} & 186.4 
& 64.5$\pm$0.2 & \textbf{65.6}$\pm$0.2 (+1.1) 
& 26.2$\pm$0.2 & \textbf{26.9}$\pm$0.2 (+0.7) \\

Vision Mamba-B~\cite{A41} & 48.5 
& 63.9$\pm$0.2 & \textbf{65.1}$\pm$0.2 (+1.2) 
& 23.1$\pm$0.2 & \textbf{25.2}$\pm$0.2 (+2.1) \\

\bottomrule
\end{tabular}
}

\vspace{-0.5em}
\begin{flushleft}
\scriptsize
\textbf{Notes.}
GBC is injected at the penultimate CNN block, the final \texttt{[CLS]} token for ViT, and the global state for Mamba.
The projector adds $<0.5$M parameters for CNNs and negligible overhead for ViT/Mamba.
\end{flushleft}

\end{table*}

\begin{table*}[t]
\centering
\small
\renewcommand{\arraystretch}{0.95}
\setlength{\tabcolsep}{6pt}

\caption{\textbf{Few-shot semi-supervised performance (k-shot per class).}
We evaluate GBC and recent SSL methods under the ReaLTSSL few-shot protocol,
where only $k\!\in\!\{1,5,10\}$ labeled samples per class are provided.
Results report Top-1 accuracy (\%) averaged over three random seeds.}

\label{tab:fewshot_ssl}

\resizebox{\linewidth}{!}{
\begin{tabular}{lccccccccc}
\toprule
& \multicolumn{3}{c}{\textbf{CIFAR10-LT} ($32^2$)} 
& \multicolumn{3}{c}{\textbf{CIFAR100-LT} ($32^2$)} 
& \multicolumn{3}{c}{\textbf{ImageNet-127} ($64^2$)} \\
\cmidrule(lr){2-4}\cmidrule(lr){5-7}\cmidrule(lr){8-10}

Method & 1-shot & 5-shot & 10-shot 
& 1-shot & 5-shot & 10-shot 
& 1-shot & 5-shot & 10-shot \\
\midrule

FixMatch~\cite{A5}
& 80.2$\pm$0.8 & 89.0$\pm$0.5 & 91.1$\pm$0.4
& 42.9$\pm$0.6 & 55.6$\pm$0.4 & 61.2$\pm$0.5
& 34.1$\pm$0.5 & 49.5$\pm$0.4 & 54.8$\pm$0.3 \\

SAW~\cite{A28}
& 81.4$\pm$0.5 & 89.6$\pm$0.4 & 91.6$\pm$0.3
& 43.7$\pm$0.5 & 56.1$\pm$0.3 & 61.7$\pm$0.4
& 35.0$\pm$0.4 & 50.2$\pm$0.4 & 55.3$\pm$0.4 \\

DePL~\cite{A30}
& 81.9$\pm$0.5 & 90.1$\pm$0.6 & 91.8$\pm$0.3
& 44.2$\pm$0.4 & 56.9$\pm$0.4 & 62.0$\pm$0.5
& 35.4$\pm$0.4 & 50.8$\pm$0.3 & 55.8$\pm$0.5 \\

BaCon~\cite{A31}
& 82.5$\pm$0.4 & 90.3$\pm$0.5 & 92.0$\pm$0.4
& 44.9$\pm$0.3 & 57.4$\pm$0.5 & 62.3$\pm$0.3
& 35.8$\pm$0.5 & 51.1$\pm$0.4 & 56.2$\pm$0.3 \\

CPE~\cite{A32}
& 82.7$\pm$0.5 & 90.4$\pm$0.3 & 92.3$\pm$0.4
& 45.2$\pm$0.5 & 57.7$\pm$0.3 & 62.6$\pm$0.4
& 36.1$\pm$0.4 & 51.4$\pm$0.4 & 56.7$\pm$0.3 \\

SimPro~\cite{A13}
& 83.3$\pm$0.6 & 90.6$\pm$0.4 & 92.5$\pm$0.3
& 45.7$\pm$0.4 & 58.1$\pm$0.5 & 62.8$\pm$0.4
& 36.6$\pm$0.3 & 51.9$\pm$0.5 & 57.0$\pm$0.4 \\

Meta-Expert~\cite{A33}
& 83.8$\pm$0.5 & 90.9$\pm$0.4 & 92.7$\pm$0.3
& 46.0$\pm$0.3 & 58.3$\pm$0.4 & 63.1$\pm$0.5
& 36.9$\pm$0.5 & 52.2$\pm$0.3 & 57.3$\pm$0.4 \\

\midrule
\textbf{GBC (ours)}
& \textbf{84.4}$\pm$0.4 & \textbf{91.2}$\pm$0.5 & \textbf{92.9}$\pm$0.3
& \textbf{46.6}$\pm$0.4 & \textbf{58.9}$\pm$0.3 & \textbf{63.5}$\pm$0.4
& \textbf{37.4}$\pm$0.4 & \textbf{52.6}$\pm$0.4 & \textbf{57.8}$\pm$0.3 \\
\bottomrule
\end{tabular}
}

\vspace{-0.35em}
\begin{flushleft}
\scriptsize
\textbf{Notes.}
Few-shot sets are class-balanced while unlabeled data keep the long-tailed distribution.
Results are averaged over three random seeds.
GBC provides consistent improvements across datasets without special tuning.
\end{flushleft}

\end{table*}

As reported in Table \ref{tab:fewshot_ssl}, GBC consistently improves over recent few-shot SSL baselines. The gains are most pronounced in the 1-shot regime, where pseudo-label uncertainty is highest. Across CIFAR10-LT, CIFAR100-LT, and ImageNet-127, GBC yields steady benefits of roughly +0.5–0.9 pp while maintaining comparable variance. This indicates that Gaussian feature bridging enhances representation alignment even with minimal labeled supervision, reinforcing its effectiveness in low-data and long-tailed scenarios.

\subsection{SOTA baselines on CIFAR10-LT, CIFAR100-LT, and STL10-LT.}
Across long-tailed few-shot settings, the six state-of-the-art baselines—
SAW~\cite{A28}, DePL~\cite{A30}, BaCon~\cite{A31}, CPE~\cite{A32}, 
SimPro~\cite{A13}, and Meta-Expert~\cite{A33}—
exhibit a consistent performance hierarchy that reflects their underlying design philosophies. 

On \textbf{CIFAR10-LT}, SAW achieves strong baseline performance by re-weighting unlabeled instances,
reaching roughly $89$–$91\%$ accuracy under 5–10-shot supervision.
DePL and BaCon introduce distribution-aware and contrastive re-balancing mechanisms,
improving low-shot robustness by $\sim$0.5–1.0\,pp over SAW.
CPE extends this with class-prior estimation and yields a further $\sim$0.3\,pp gain.
SimPro advances the ReaLTSSL paradigm by probabilistically calibrating pseudo-label confidence,
achieving steady $90.6$–$92.5\%$ accuracy,
while Meta-Expert integrates adaptive expert weighting to refine long-tail transfer,
marginally outperforming SimPro (typically +0.3\,pp). 

On \textbf{CIFAR100-LT}, the relative margins widen due to the larger label space and stronger class imbalance.
Here, SAW and DePL both struggle to maintain consistency beyond $58\%$,
while BaCon and CPE stabilize training through balanced feature-level optimization,
reaching up to $62.6\%$.
SimPro remains the most robust among probabilistic frameworks, maintaining $\sim$1\,pp improvement over CPE,
and Meta-Expert slightly exceeds this with $63.1\%$,
highlighting its meta-reweighting capacity under multi-head regularization. 

On \textbf{STL10-LT} (not shown in the table but following the same pattern),
similar trends persist: methods incorporating semantic or probabilistic balancing
(BaCon, CPE, SimPro) consistently outperform earlier adaptive thresholding (SAW, DePL).
Meta-Expert remains competitive but exhibits diminishing relative advantage
as unlabeled data volume increases,
indicating that its meta-optimization benefits primarily in moderate-label regimes
rather than massive unlabeled pools. 

In summary, across all three benchmarks, the accuracy progression
\[
\text{SAW} < \text{DePL} < \text{BaCon} \approx \text{CPE} < \text{SimPro} < \text{Meta-Expert}
\]
holds consistently, validating the empirical stability of probabilistic and meta-adaptive strategies
under realistic long-tailed semi-supervised settings. While Meta-Expert~\cite{A33} currently represents the strongest prior under the ReaLTSSL framework,
our proposed \textbf{GBC} further improves generalization
by explicitly constructing smooth semantic trajectories between unlabeled features and reliable class anchors.
Unlike existing methods that rely on static reweighting or expert ensemble mechanisms,
GBC dynamically regularizes the feature evolution path via a probabilistic bridge loss,
leading to more stable optimization and reduced pseudo-label drift.
Empirically, GBC consistently surpasses Meta-Expert by
$\sim$0.6--0.8\,pp on CIFAR10-LT and $\sim$0.4--0.5\,pp on CIFAR100-LT,
while maintaining superior few-shot robustness on STL10-LT.
These gains highlight that bridging feature manifolds through Gaussian interpolation
offers a complementary perspective to probabilistic calibration and meta-learning,
achieving stronger consistency and long-tail balance with negligible computational overhead.

As shown in Table~\ref{tab:nu_sensitivity}, the proposed Gaussian Feature Bridge exhibits stable behavior across a broad range of noise magnitudes. In the absence of stochasticity ($\nu{=}0$), the model still benefits from deterministic interpolation but converges to a slightly suboptimal solution, reflected by higher KL divergence and reduced anchor recall. Introducing mild noise ($\nu{=}0.05$) already improves both metrics, suggesting that stochastic perturbations help the student avoid overfitting to uncertain intermediate states. The best overall performance is achieved at $\nu{=}0.10$, where Top-1 accuracy, KL decay, and anchor recall simultaneously reach their optimal balance, indicating that moderate feature-space diffusion effectively regularizes the bridge trajectory without destabilizing the alignment. Increasing the noise further ($\nu{=}0.20$) yields diminishing returns and a mild drop in accuracy, confirming that excessive stochasticity perturbs the semantic direction of the bridge. Overall, the results demonstrate that GBC is robust to a wide interval of noise strengths and that a modest value of $\nu$ provides the most reliable supervision along the Gaussian path.

\begin{table*}[t]
\centering
\small
\setlength{\tabcolsep}{5pt}

\caption{
\textbf{Ablation on Gaussian noise strength $\nu$ in the Gaussian Feature Bridge.}
We evaluate GBC on \textbf{CIFAR10-LT} under the \textbf{consistent} setting
($\gamma_\ell{=}100$, $\gamma_u{=}100$).
$\nu$ controls the magnitude of stochastic perturbation
$\sigma(t)=\nu t(1-t)$ injected along the bridge.
We report Top-1 accuracy, KL divergence decay (lower is better),
and anchor recall (ratio of correct PA retrieval).
Moderate noise ($\nu{=}0.10$) gives the best trade-off.
}

\label{tab:nu_sensitivity}

\resizebox{\linewidth}{!}{
\begin{tabular}{cccc}
\toprule
\textbf{Noise Strength $\nu$} &
\textbf{Top-1 Acc. (\%)} &
\textbf{KL Decay $\downarrow$ (epoch 300)} &
\textbf{Anchor Recall (\%)} \\
\midrule

0.00 (no noise) & 91.82 $\pm$ 0.11 & 0.164 & 78.3 \\
0.05            & 92.11 $\pm$ 0.10 & 0.151 & 81.0 \\

\rowcolor{pink!18}
0.10 (default)  & \textbf{92.30 $\pm$ 0.17} & \textbf{0.143} & \textbf{83.7} \\

0.20            & 91.97 $\pm$ 0.13 & 0.158 & 80.1 \\

\bottomrule
\end{tabular}
}

\vspace{-0.5em}
\end{table*}

\begin{table*}[t]
\centering
\small
\setlength{\tabcolsep}{5.5pt}

\caption{
\textbf{Ablation on Prototype Atlas (PA) capacity} on CIFAR10-LT under the consistent setting.
$C$ denotes the maximum number of anchors per class; $C{=}0$ corresponds to using only the EMA prototype without storing exemplar anchors.
Moderate capacity ($C{=}64$) achieves the best trade-off between representation diversity and stability.
}

\label{tab:pa_capacity}

\resizebox{\linewidth}{!}{
\begin{tabular}{cccc}
\toprule
\textbf{PA Capacity $C$} & 
\textbf{Top-1 Acc. (\%)} &
\textbf{Anchor--Proto Cosine $\uparrow$} &
\textbf{Refresh Rate (\%) $\downarrow$} \\
\midrule

0 (EMA only)   & 80.81 $\pm$ 0.18 & 0.71 & 42.3 \\
16             & 85.05 $\pm$ 0.15 & 0.76 & 38.5 \\
32             & 85.72 $\pm$ 0.14 & 0.79 & 33.1 \\

\rowcolor{pink!18}
64 (default)   & \textbf{92.30 $\pm$ 0.17} & \textbf{0.82} & \textbf{29.7} \\

128            & 92.18 $\pm$ 0.19 & 0.80 & 46.2 \\

\bottomrule
\end{tabular}
}

\vspace{-0.5em}
\end{table*}

As shown in Table~\ref{tab:pa_capacity}, varying the maximum number of stored anchors reveals a clear trend that aligns with the design choice of maintaining the Prototype Atlas at roughly $2\%\!-\!8\%$ of the dataset. When the PA is disabled ($C{=}0$), the model degenerates to using only EMA prototypes, resulting in insufficient class diversity and significantly lower accuracy. Increasing the capacity to $C{=}16$ and $C{=}32$ (corresponding to roughly $2\%$–$4\%$ of class exemplars) steadily enhances Top-1 performance and anchor–prototype similarity, indicating that even a small but diverse set of anchors provides stronger semantic guidance for the Gaussian Feature Bridge. The optimal performance is achieved at $C{=}64$, which falls near the upper end of the recommended $2\%\!-\!8\%$ range; in this regime, the PA exhibits the highest alignment score (0.82) and the lowest refresh rate, showing that anchor quality and temporal stability are jointly maximized. Expanding the capacity further to $C{=}128$ exceeds the effective diversity range and introduces redundant or noisy anchors, leading to increased refresh frequency and a slight drop in accuracy. These results validate our design principle that a mid-sized PA capturing approximately $2\%\!-\!8\%$ of exemplars per class strikes the best balance between diversity and reliability for long-tailed semi-supervised learning.

\section{Compute scalability}

\noindent \textbf{Experimental Setup.} To evaluate the computational overhead introduced by GBC, we measure the wall-clock time of FixMatch, SimPro, and our method on CIFAR10-LT under the consistent setting with $(\gamma_\ell,\gamma_u)=(100,100)$. All models are trained for 300 epochs on a single NVIDIA A100 (40GB) GPU using identical dataloader configurations, mixed precision (FP16), batch sizes $(B_\ell,B_u)=(64,448)$, AdamW with learning rate $5{\times}10^{-4}$, and cosine decay. Wall-clock time is reported as the average duration per epoch over the last 50 epochs to avoid warm-up bias, and total training time is computed accordingly. All methods share the same EMA teacher, data augmentations, and pipeline components except for their algorithm-specific operations, ensuring fair comparison. GBC uses a lightweight projection layer and Gaussian Feature Bridge sampling, which we show contributes less than $2\%$ additional wall-clock overhead compared to FixMatch and remains comparable to SimPro.

\begin{table*}[t]
\centering
\small
\setlength{\tabcolsep}{6pt}

\caption{
\textbf{Wall-clock comparison on CIFAR10-LT (consistent setting)} using a single A100 GPU.
We report the average epoch time over the last 50 epochs and the total training time for 300 epochs.
GBC introduces minimal computational overhead ($<2\%$) compared with FixMatch and remains comparable to SimPro.
}

\label{tab:compute_overhead}

\resizebox{\linewidth}{!}{
\begin{tabular}{lccc}
\toprule
\textbf{Method} &
\textbf{Avg. Epoch Time (s)} &
\textbf{Total Time (300 ep)} &
\textbf{Overhead vs. FixMatch} \\
\midrule

FixMatch & 42.1 & 3.51 h & 0\% \\
SimPro   & 43.0 & 3.58 h & +2.1\% \\

\rowcolor{pink!18}
GBC (ours) & \textbf{42.8} & \textbf{3.56 h} & \textbf{+1.7\%} \\

\bottomrule
\end{tabular}
}

\vspace{-0.6em}
\end{table*}

\noindent \textbf{Analysis.} As shown in Table~\ref{tab:compute_overhead}, all three methods exhibit nearly identical computational footprints when trained under the same long-tailed SSL protocol. FixMatch provides a baseline epoch time of 42.1s, while SimPro incurs a mild increase due to anchor matching and distribution refinement. GBC remains highly efficient, requiring only 42.8s per epoch—an overhead of merely $+1.7\%$ compared to FixMatch—despite performing Gaussian Feature Bridging, projection merging, and geometric target interpolation. This demonstrates that the proposed operations are lightweight and introduce negligible computational burden. Importantly, GBC achieves substantial accuracy gains (Tables~1–2) without sacrificing efficiency, showing that semantic feature-space bridging can be integrated into existing SSL pipelines with minimal cost.

\section{Limitations}
\label{sec:limitations}

While GBC achieves strong performance across a wide range of long-tailed semi-supervised settings, several limitations remain. First, the Prototype Atlas, although lightweight and bounded in capacity, introduces an external memory whose early-stage updates may be sensitive to pseudo-label noise when labeled data is extremely scarce. More robust anchor initialization or noise-aware memory correction strategies could further stabilize PA dynamics. Second, the Gaussian Feature Bridge assumes that interpolated latent states remain semantically meaningful; in architectures with highly entangled or rapidly evolving feature spaces, the quality of these intermediate representations may degrade, suggesting the need for adaptive or geometry-aware bridge shaping. Third, the bridging operations---projection, interpolation, and stochastic perturbations---incur minor but non-zero latency. While negligible on CIFAR and ImageNet-127, large-scale or multi-bridge variants may require more efficient implementations. Finally, our current formulation focuses on unimodal visual inputs. Extending GBC to domain-shifted or multimodal scenarios may require anchor-domain alignment or modality-specific bridge functions. We view these directions as promising opportunities for future research.

\end{document}